\documentclass[12pt]{article}

\usepackage{geometry,setspace}
\usepackage{fullpage,xcolor,soul,url,natbib}
\usepackage{graphicx}
\usepackage{subfigure}
\usepackage{booktabs} 
\usepackage{amsmath}
\usepackage{amssymb}
\usepackage{mathtools}
\usepackage{amsthm}

\usepackage{hyperref}

\hypersetup{
           breaklinks=true,   
           colorlinks=false,   
           pdfusetitle=true,  
        }

\theoremstyle{plain}
\newtheorem{theorem}{Theorem}[section]

\newtheorem{lemma}[theorem]{Lemma}

\theoremstyle{definition}

\newtheorem{assumption}[theorem]{Assumption}
\theoremstyle{remark}

\usepackage[textsize=tiny]{todonotes}

\def\max{\text{max}}
\def\min{\text{min}}
\def\tilde{\widetilde}
\def\hat{\widehat}
\newcommand{\norm}[1]{\left\lVert#1\right\rVert}

\def\Xi{ A}
\def\Upsilon{ g}

\begin{document}

\title{Least Squares Estimation using Sketched  Data with Heteroskedastic Errors}

\author{\textsc{Sokbae Lee}\thanks{Department of Economics, Columbia University and Institute for Fiscal Studies} \and \textsc{Serena Ng}\thanks{Department of Economics, Columbia University and NBER\newline
The authors would like to thank David Woodruff and Shusen Wang for helpful discussions.
The first author would like to thank the European Research Council for financial support (ERC-2014-CoG- 646917-ROMIA) and the UK Economic and Social Research Council for research grant (ES/P008909/1) to the CeMMAP.
The second author would like to thank the National Science Foundation for financial support  (SES: 2018369).}
}

\maketitle

\begin{abstract}
Researchers may perform regressions  using a sketch of data of size $m$ instead of the full sample of size $n$ for a variety of reasons.  This paper considers the case when the regression errors do not have constant variance and heteroskedasticity robust standard errors would normally be needed for test statistics to provide  accurate inference. We show that  estimates using   data sketched by random projections will  behave `as if' the errors were homoskedastic. Estimation by random sampling would not have this property.    The result arises because the sketched estimates in the case of random projections  can be expressed as   degenerate $U$-statistics, and  under certain conditions, these statistics are  asymptotically normal with  homoskedastic variance.   We verify that the conditions hold not only in the case of least squares regression when the covariates are exogenous, but also in instrumental variables estimation when the covariates are endogenous. The result implies that inference, including first-stage F tests for instrument relevance,  can be simpler than the full sample case if the sketching scheme is appropriately chosen.
\bigskip


\end{abstract}

\bibliographystyle{economet}

\thispagestyle{empty}
\setcounter{page}{0}
\newpage

\onehalfspacing

\section{Introduction}

Big data sets can be costly to store and analyze, 
and one approach  around the data bottlenecks is to work with 
a randomly chosen subset, or a \emph{sketch}, of the data. Data privacy may also dictate that a sketch of the data be made available for public use.  The early works of \citet{sarlos-06}, \citet{dmm:06} and \citet{dmms} consider the algorithmic properties of  the least squares estimator using sketched data. Subsequent work extends the  analysis  to
ridge regression \citep[e.g.,][]{wang2017sketched,Liu2020Ridge},
and logistic regression \citep[e.g.,][]{wang2019more}. See, e.g.,  \citet{woodruff2014sketching}, \citet{drineas2018lectures} and \citet{martinsson2020randomized} for a review.  However,   \citet{ma2015statistical}, \citet{raskutti2016statistical}, \citet{dobriban2019asymptotics} and \citet{ma2020asymptotic} have found that an optimal worse-case (algorithmic) error  may not yield  an optimal  mean-squared (statistical) error.  This has led to  interest  in better understanding the  sketched least squares estimates in a Bayesian setting as in   \citet{geppert-etal:17}, or its asymptotic  distribution  as in \citet{ahfock-astle-richardson} and \citet{ma2020asymptotic}.    \citet{LN:2020}  highlights the tension between  a large $m$  required for accurate inference, and a small $m$ for computation  efficiency. 
 To date, these results have been derived under the  assumptions that the errors are   homoskedastic  and that the regressors are exogenous.  But these assumptions are not innocuous. The estimates will be biased when the regressors are not exogenous, as would normally be the case in causal inference. And if the errors are heteroskedastic,  test statistics must use standard errors robust to heteroskedasticity, or else inference will not be accurate even if the regressors are exogenous. 

In this  paper,   we obtain the  surprising result that when sketching is based on random projections,  robust standard errors will not be needed, meaning that  inference using the sketched estimates can proceed  as though the errors were homoskedastic. The proof is obtained by analyzing the difference between the full sample and the sketched estimates in terms of degenerate $U$-statistics. However, the result does not hold when sketching is based on  random sampling.  Our analysis of the least squares estimator and two-stage least squares estimator shows that these findings hold both when the regressors are exogenous and  endogenous.  An implication is that inference, including first-stage F tests for instrument strength, may not require heteroskedasticity robust standard errors if the estimates are based on appropriately sketched data.


The following notation will be used.  Let $\norm{ a }$ denote the Euclidean norm of any vector $a$. 
Let $A_{ij}$ or $[A]_{ij}$ denote the $(i,j)$ element of a matrix $A$.
 For $k = 1,\ldots,d$, let $\sigma_k(A)$ be a singular value of $A$. Let $\norm{ A }_2= \sigma_{\max}(A)$ denote its spectral norm,  where
$\sigma_{\max}(A)$,  and $\sigma_{\min}(A)$ are  the largest and smallest singular values of $A$.  The superscript $T$ denotes the transpose of a matrix.
For an integer $n \geq 1$, 
 $[n]$ is the set of positive integers from 1 to $n$. 
Let $\rightarrow_p$ and $\rightarrow_d$, respectively,  denote convergence in probability and in distribution.
 For a sequence of random variables $A_n$ and a sequence of positive real numbers $a_n$, 
$A_n = o_p(a_n)$ iff $ a_n^{-1} A_n \rightarrow_p 0$;
 $A_n = O_p(a_n)$ iff $a_n^{-1} A_n$ is bounded in probability.

An accompanying R package is available on the Comprehensive R Archive Network (CRAN) at \url{https://CRAN.R-project.org/package=sketching} and all replication files are available at \url{https://github.com/sokbae/replication-LeeNg-2022-ICML}.


\section{Sketched Least Squares Estimation with Heteroskedastic Errors}

Given $n$ observations $\{ (y_i, X_i,Z_i): i=1,\ldots,n \}$, we consider  a linear regression model:
\begin{align}\label{model}
y_i = X_i^T \beta_0 + e_i, \; 
\end{align}
where $y_i$ is the scalar dependent variable,
$X_i$ is a $p \times 1$ vector of regressors,   $\beta_0$ is a $p \times 1$ vector of unknown parameters. The innovation $e_i$ is said to be (conditionally)  homoskedastic if $E[e_i^2|X_i]=E[e_i^2]$.
Otherwise, $e_i$ is said to be heteroskedastic.  The regressors are  said to be exogenous if $E[e_iX_i]=0$. Otherwise it is endogenous. In that case, we assume  a $q\times 1$ vector of instrumental variables, $Z_i$, satisfying $E[e_iZ_i]=0$ are available.
In matrix form, the model given in \eqref{model} can be written as
\begin{align*}
y = X \beta_0 + e, 
\end{align*} 
where $y$ and $e$ are  $n \times 1$ vectors whose $i$-th rows are $y_i$ and $e_i$, respectively,
and 
$X$ is the $n \times p$ matrix of regressors whose $i$-th row is $X_i^T$. 

We first study  the exogenous regressor case  when $\mathbb E(e_iX_i)=0$. The least squares estimator $\hat\beta_{OLS}:=(X^TX)^{-1} X^T y$ is $\sqrt{n}$ consistent and asymptotically normal, i.e.,
$ \sqrt{n}(\hat\beta_{OLS}-\beta_0)\rightarrow_d N(0,V_1)$ as $n\rightarrow\infty$, where \[V_1 :=   [\mathbb E(X_iX_i^T) ]^{-1} \mathbb E(e_i^2X_iX_i^T) [\mathbb E(X_iX_i^T)]^{-1}\]
 is the  heteroskedasticity-robust asymptotic variance.
 Under homoskedasticity,  $V_1$ becomes 
\[  V_0:= \mathbb E(e_i^2) [\mathbb E(X_iX_i^T)]^{-1}.\]
The point estimates $\hat\beta$ can be used to test hypothesis, say, $H_0:\beta_2=\bar \beta_2$ using the $t$ test    $\frac{\sqrt{n}(\hat\beta_2-\bar \beta_2)}{\sqrt{[\hat V]_{22}}}$, where $\hat V$ is an estimate of either $V_1$ or $V_0$, $\beta_2$ is a specific element of $\beta_2$,
$\bar \beta_2$ is the null value, and $[\hat V]_{22}$  the (2,2)  diagonal element of $\hat V$.
The distribution of this test under the null hypothesis crucially  depends on the correct standard error $\sqrt{[\hat V]_{22}}$ being used.   Using $\hat V_0$ when the robust estimator $\hat V_1$ should have been used would lead to inaccurate inference, in the sense of rejecting the null hypothesis too often or not enough.

A sketch of the data $(y, X)$ is  $(\tilde{y}, \tilde{X})$, where
$\tilde{{y}} = \Pi {y}$, $\tilde{{X}} = \Pi {X}$,  and 
 $\Pi$ is usually an  $m \times n$  random matrix.
The sketched least squares estimator is $\tilde\beta_{OLS}:=(\tilde X^T\tilde X)^{-1}\tilde X^T \tilde y$. 
Even though the sketched regression is based on a sample of size $m$, 
$\tilde X^T\tilde X = X^T \Pi^T \Pi X$ and $\tilde X^T \tilde y = X^T \Pi^T \Pi y$ can be seen as weighted moments in a sample of size $n$.  
Thus let $\tilde{\Upsilon}_n := \tilde{{X}}^T   \tilde{{e}}/n$,
$\hat \Upsilon_n := {X}^T {e}/n$,
$\tilde{\Xi}_n := (\tilde X^T\tilde X/n)^{-1}$, and
$\hat{\Xi}_n :=  ( X^T X/n)^{-1}.$  Then
\begin{align*}
\tilde{\beta}_{OLS} - \hat{\beta}_{OLS} 
&= (\tilde{\Xi}_n - \hat{\Xi}_n) \hat\Upsilon_n + \hat{\Xi}_n ( \tilde{\Upsilon}_n - \hat \Upsilon_n) \\
&+ (\tilde{\Xi}_n - \hat{\Xi}_n) ( \tilde{\Upsilon}_n -\hat \Upsilon_n), 
\end{align*}
By the law of large numbers,
$\hat \Xi_n - \Xi = o_p (1)$, where $\Xi := [ \mathbb{E} ( X_i X_i^T ) ]^{-1}$,  
and by the central limit theorem, $\hat g_n=O_p(n^{-1/2})$. 
We  show in Section \ref{sec:mse} that for 
$\Pi$  with subspace embedding property
 \begin{align*}
\tilde{\beta}_{OLS} - \hat{\beta}_{OLS} 
&= \Xi ( \tilde{\Upsilon}_n - \hat \Upsilon_n)
+ o_p ( m^{-1/2}).
\end{align*}
We study  $\tilde \beta_{OLS}$ under the following regularity conditions.

\begin{assumption}
\label{OLS-textbook}
\begin{itemize}
\item[(i)] The data $\mathcal{D}_n := \{(y_i, X_i) \in \mathbb{R}^{1+p}
: i=1,\ldots,n \}$ are independent and identically distributed (i.i.d.), where $p$ is fixed.  Furthermore,
$X$ has  singular value decomposition ${X} = U_X \Sigma_X V_X^T$.
\item[(ii)] 
 $\mathbb{E}(y_i^4)<\infty$, 
$\mathbb{E}(\| X_i \|^4)<\infty$,
and
$\mathbb{E} ( X_i X_i^T )$   has full rank $p$.
\item[(iii)] 
The random matrix $\Pi$ is independent of $\mathcal{D}_n$. 
\item[(iv)] 
$m = m_n \rightarrow \infty$ but $m/n \rightarrow 0$ as $n \rightarrow \infty$.
\end{itemize}
\end{assumption}

Assumptions (i) and (ii) are standard. For (iii),  note that for a general random  $\Pi$ whose $(k,i)$ element
 is  $\Pi_{ki}$,   the difference between the  full and the sketched moments  such as $\tilde g_n-\hat g_n$ and $\tilde A_n-\hat A_n$  are of the form
\begin{align*}
\begin{split}
&n^{-1} \left( U^T \Pi^T \Pi V - U^T V \right) \\
&= n^{-1} \sum_{i=1}^n    \psi_i  U_i    V_i
+ n^{-1} \sum_{i=1}^n  \sum_{j=1, j \neq i}^n  U_i   \varphi_{ij} V_j \\
&=: T_{n1} + T_{n2},
\end{split}
\end{align*}
where
$U \in \mathbb{R}^n$ and $V \in \mathbb{R}^n$ are  vectors of certain
i.i.d. random variables $(U_i, V_i) \in \mathbb{R}^2$ that are independent of $\Pi$,
\begin{align*}
\psi_i := \sum_{k=1}^{\textrm{r.dim}(\Pi)} \Pi_{ki}^2  - 1,
\; 
\varphi_{ij} := \sum_{k=1}^{\textrm{r.dim}(\Pi)} \Pi_{ki}  \Pi_{kj},
\end{align*}
and $\textrm{r.dim}(\Pi)  \in \{m,  n\}$ denotes the row dimension of $\Pi$.

There are two classes of sketching schemes to consider. Random sampling schemes have  $\varphi_{ij} = 0$ for all $i \neq j$ because  there is only one non-zero entry in  each row of $\Pi$. In such cases, $T_{2n}$ is negligible and $T_{1n}$ is the leading term.  The second class is  random projection schemes with which $T_{1n}$ is asymptotically negligible and $T_{2n}$ is the leading term. 

To gain intuition,  we first provide results for Bernoulli sampling (BS) from the first type and countsketch (CS) from the second type.
\begin{theorem}\label{corr-CLT-thm}
Let Assumption~\ref{OLS-textbook} hold and $\mathbb E(e_iX_i)=0$.
\begin{itemize}
\item[(i)] Under BS, $ m^{1/2} ( \tilde{\beta}_{OLS} - \hat{\beta}_{OLS} \,)  \rightarrow_d N (0, V_1 )$.
\item[(ii)] Under CS,
$m^{1/2} ( \tilde{\beta}_{OLS} - \hat{\beta}_{OLS} \,)  
\rightarrow_d N (0, V_0 )$.
\end{itemize}
\end{theorem}

Though Theorem \ref{corr-CLT-thm} indicates that both sampling schemes yield  asymptotically normal estimates, their variances are different, and normality  holds  for different reasons. The proof is given in the Appendix. Here, we sketch the main arguments.

 First, under BS, the sampling probability is determined by i.i.d. Bernoulli random variables with 
success probability $m/n$. Thus, $\Pi = \sqrt{\frac{n}{m}} B$ is an $n\times n$ matrix (not $m\times n$), where $B$ is a diagonal sampling matrix. We have 
\begin{align*}
\frac{1}{n} \left( X^T \Pi^T \Pi e - X^T e \right) 
&= \frac{1}{n} \sum_{i=1}^n    \left( \frac{n}{m} B_{ii} - 1 \right) X_i    e_i=T_{n1}.
\end{align*}
Since the summands are i.i.d. with mean zero  with  variance 
$( \frac{n}{m} - 1 ) \mathbb{E} ( e_i^2  X_i X_i^T)$, applying
 central limit theorem for i.i.d. observations yields the sandwich variance $V_1$.

Consider now CS.  Each column of its $\Pi$ has one non-zero entry  taking on value
 $\{ +1, -1\}$ randomly drawn with equal probability and  located  uniformly at random.  For such $\Pi$ and every nonzero $c \in \mathbb{R}^p$, 
\begin{align*}
&c^T(X^T \Pi^T \Pi e - X^T e) / n \\
& =
n^{-1} \sum_{i=1}^n   \bar{X}_j(c)  \left( \sum_{k=1}^{m}  \Pi_{ki}^2 - 1 \right)   e_{i} \\
& \;\;\; + 
n^{-1} \sum_{i=1}^n  \sum_{j=1, j \neq i}^n \sum_{k=1}^{m} \bar{X}_j(c)   \Pi_{kj} \Pi_{ki}  e_{i} \\
& =  T_{n1} + T_{n2},
\end{align*}
where $\bar{X}_i(c) :=  \sum_{u}^q c_u X_{iu}$ is a weighted sum of elements of the $i$-th row of $X$. 
The term  $T_{n1}$ is identically zero  because  there is only one non-zero entry per column of $\Pi$.  

To analyze $T_{n2}$,
let $W_i = (Y_i, X_i^T, \Pi_{1i},\ldots, \Pi_{mi})^T$. Since the columns of $\Pi$ are i.i.d.,   $\{W_i : i=1,\ldots,n \}$ are i.i.d.
Now let  $w = (y, x^T, \pi_1, \ldots, \pi_m)^T$ be a non-random index. 
Define
$
\tilde{H} (w_1, w_2) : =  \sum_{k=1}^{m} \bar{x}_{1}(c)  \pi_{k1}  S\pi_{k2} e_{i}
$
and
$
H (w_1, w_2):= \tilde{H} (w_1, w_2) + \tilde{H} (w_2, w_1).
$
We can write
\begin{align*}
T_{n2} 
&= n^{-1} \operatorname*{\sum \sum}_{1 \leq i < j \leq n}  H (W_i, W_j),
\end{align*}
noting that $H (w_1, w_2) = H (w_2, w_1)$,
and $\mathbb{E} ( H (W_1, W_2) | W_1) = \mathbb{E} ( H (W_1, W_2) | W_2) = 0$. Importantly,  $T_{n2} $ has now been represented as a degenerate $U$-statistic. In general, the asymptotic distribution of such statistics is either a weighted average of independent, centered chi-square random variables with  complex   weights, or  a centered normal distribution. 
 But if the conditions  given in \citet{hall1984central} are satisfied, the latter holds. Precisely,
\begin{align*}
\left\{ \frac{1}{2} \mathbb{E} [ H^2 (W_1, W_2)  ] \right\}^{-1/2}  T_{n2} \rightarrow_d N(0, 1).
\end{align*}
A  sufficient condition for this result which we verify in the Appendix is
\begin{align*}
\frac{\mathbb{E}[ G^2 (W_1, W_2) ]  + n^{-1}  \mathbb{E} [ H^4 (W_1, W_2)  ] }
{ \{ \mathbb{E} [ H^2 (W_1, W_2)  ] \}^2 } \rightarrow 0 \ \ \text{ as $n \rightarrow \infty$},
\end{align*}
where $ G (w_1, w_2) := \mathbb{E} [ H (W_1, w_1) H (W_1, w_2)  ]$.
Furthermore, we also verify that  for $W_i \neq W_j$, 
$$
\frac{1}{2} \mathbb{E} [ H^2 (W_i, W_j)  ]
= \frac{1}{m} \mathbb{E} ( \bar{X}^2_i(c))  \mathbb{E} (e_i^2 ).
$$

Note that  $\mathbb E (e_i^2)$ appears separately from $\mathbb E(\bar X_i^2(c))$. This is key to the claim in Theorem~\ref{corr-CLT-thm}   that when $\tilde \beta$ is based on CS, $m^{1/2}(\tilde\beta-\hat\beta)\rightarrow_d N(0,E[e_i^2] \Xi)=N(0,V_0) $. Analogous arguments show that each entry of $(\tilde X^T \tilde X- X^T X)$ can also be written as a degenerate $U$-statistic and  $\tilde \Xi_n-\hat \Xi=o_p(1)$, which is needed for consistent estimation of $V_0$ and $V_1$. 

As discussed in  \citet{charikar2004finding} and \citet{clarkson2013low,clarkson2017low}, the sparsity of $\Pi$ significantly reduces  the run time required of the countsketch  to compute $\Pi A$ to O(nnz(A)), where nnz(A) is the number of non-zero entries of $A$.  Another appeal of countsketch is that the sketches can be obtained  by streaming without constructing $\Pi$. Here, we show that  countsketch  removes heteroskedasticity which is appealing because it simplifies inference.  In the next section,  we study the mean-squared sketching error and  show that part (i) of Theorem  \ref{corr-CLT-thm} also holds for other $\Pi$s in  the first class,  while part (ii) holds for other $\Pi$s in the second class.   Section~\ref{sec:2sls} then shows that these  results  also hold when the regressors are not exogenous.

\section{The Mean-Squared Sketching Error}\label{sec:mse}

For a random variable $G$, let $\mathrm{MSE} ( G ) = [\mathbb{E} (G)]^2 + \mathrm{Var} (G)$ denote the mean squared error. 
We now analyze the asymptotic behavior of mean squared sketching errors of 
$(U^T \Pi^T \Pi V - U^T V)/n$, where 
 $U \in \mathbb{R}^n$ and $V \in \mathbb{R}^n$ denote vectors of 
i.i.d. random variables $(U_i, V_i) \in \mathbb{R}^2$ that are independent of $\Pi$, with
 $\mathbb{E} (U_i^4) < \infty$ and $\mathbb{E} (V_i^4) < \infty$. 
Recall  that $m = m_n \rightarrow \infty$ but $m/n \rightarrow 0$ as $n \rightarrow \infty$.

\subsection{Random Sampling with Replacement (RS)}


For sketching by random sampling with replacement (RS),  we suppose that
for each $t = 1,\ldots,m$, 
we sample $k_t$ from $[n]$ with probability $p_i := \mathrm{Pr}(k_t = i)$ independently and with replacement.
The random matrix $\Pi \in \mathbb{R}^{m \times n}$ is then
\begin{align*}
\Pi = 
\sqrt{\frac{n}{m}} 
\begin{pmatrix}
\iota_{k_1} &
\ldots & 
\iota_{k_m}
\end{pmatrix}^T,	
\end{align*}
where $\iota_k$ is the $k$-th column vector of the $n \times n$ dimensional identity matrix. Sketching schemes of the RS class have properties characterized by Lemma \ref{moments-unif-lem} in the Appendix. Importantly, each $\Pi$ in this class has $T_{2n}=0$ and as a result, $T_{n1}$ is the only term we need to consider.
An important example in the RS class is  uniform sampling with replacement with $p_i = n^{-1}$.



\begin{theorem}\label{moments-thm}
(i) If  $\Pi$ is a random matrix satisfying {RS} and $\sum_{i=1}^n p_i^2 = o (m^{-1})$,
then, as $n \rightarrow \infty$, 
\begin{align*}
\mathrm{MSE} \left[ \sqrt{m} \left( U^T \Pi^T \Pi V - U^T V \right) / n \right] 
\rightarrow \mathrm{Var}( U_i V_i ).
\end{align*}
(ii) If $\Pi$ is Bernoulli sampling matrix (BS), 
then, as $n \rightarrow \infty$,
\begin{align*}
&\mathrm{MSE} \left[ \sqrt{m} \left( U^T \Pi^T \Pi V - U^T V \right) / n \right] 
\rightarrow \mathbb{E} ( U_i^2 V_i^2 ).
\end{align*}
\end{theorem}
The mean-squared errors for RS and BS are the same if $U_i V_i$ is mean zero.

Theorem~\ref{moments-thm} is useful in two ways.
First, let $U = c^T X^T$ and $V = X c$.  Under Assumption~\ref{OLS-textbook} and for all nonzero $c \in \mathbb{R}^p$,  Theorem~\ref{moments-thm} yields
\begin{align*}
\mathrm{MSE} \left[ \sqrt{m} (c^T \tilde{X}^T \tilde{X} c  - c^T X^T X c) / n \right] = O(1).
\end{align*}

By Chebyshev's Inequality, 
$( \tilde X^T\tilde X - X^T X ) /n  = O_p ( m^{-1/2} )$.
Similarly, for  $U=X c$ and $V=e$, the theorem implies
$\tilde{\Upsilon}_n - \hat \Upsilon_n = O_p ( m^{-1/2} )$.
Applying continuous mapping theorem gives the result stated earlier that
$\tilde{\beta}_{OLS} - \hat{\beta}_{OLS} 
= \Xi ( \tilde{\Upsilon}_n - \hat \Upsilon_n)
+ o_p ( m^{-1/2} )$.

It is known that   the efficient estimator under heteroskedasticity is  the generalized least squares (GLS), defined as
 \[\hat\beta_{GLS}=(X^T\Omega^{-1} X)^{-1} X^T\Omega^{-1} y,\] 
where $\Omega$ is $n\times n$ a diagonal matrix with $\Omega_{ii}=\sigma^2_i$. GLS  weights each observation  with $\sigma_i^{-1}$ so that the  errors in the weighted regression are homoskedastic. Now the OLS estimator applied to sketched data can be written as 
\[ \tilde\beta_{OLS}= (X^T\Pi^T\Pi X)^{-1} X^T\Pi^T \Pi y.\]
A question of interest is whether $\Pi^T\Pi$ can play the role of $\Omega^{-1}$.  The theorem sheds light on this problem as its second use    is to obtain the asymptotic variance of the sketched estimator.
To this end, let  again  $U_i=c^T X_i$ and $V_i=e_i$.  Assuming $\mathbb E(g_i)=0$ where $g_i=e_i X_i$,   RS and  BS imply:
\begin{align*} 
\mathrm{MSE} \left[ \sqrt{m} \, c^T( \tilde{\Upsilon}_n -\hat \Upsilon_n) \right] =
\mathbb{E} ( e_i^2 c^T X_i  X_i^T c ). 
\end{align*}
The asymptotic standard error is generally the expectation of a product of $e_i^2$ and $(c^TX_i)^2$ and becomes the product of two expectations only under homoskedasticity when $\mathbb E[e_i^2 c^T X_iX_i^T c]=\mathbb E[e_i^2] \mathbb E[c^T X_iX_i^Tc]$. Thus, under RS and BS, the asymptotic variance of $\tilde\beta_{OLS}$ is $V_0$ only if homoskedasticity is explicitly imposed, implying that $\Pi^T\Pi$  corresponding to  random sampling  will not homogenize error variance in the same way that $\Omega^{-1}$ can.

It is noteworthy that even under homoskedsaticity,  we cannot always use a central limit theorem for i.i.d. data even if the full sample of data are i.i.d. because   the sampling scheme   may  induce dependence in the sketched data. Thus the asymptotic normality result can  only be analyzed on a case by case basis.   \citet{ma2020asymptotic} confronts a similar problem when studying the asymptotic distribution of  estimators in
 linear regressions under random sampling with replacement and  homoskedastic errors. Let 
$K_{i}$ and $p_i$, respectively, denote the number of times and the probability that $i^{t h}$ observation is sampled. 
Their estimator has $W =\operatorname{diag}\left\{K_{i} / ( m p_i ) \right\}_{i=1}^{n}$ playing the role of $\Pi^T\Pi$. Our Theorem 3.1 applies to their setup with  uniform sampling where $p_i=1/n, n\ge m$, but it would not apply  when $p_i$ is  data dependent. In this case, \citet{ma2020asymptotic}  also cannot use a central limit theorem for i.i.d. data. Instead, they  apply Hay\'ek-Sidak central limit theorem and use  Poissonization to account for dependence in the sketched data that arises from sampling.


\subsection{Random Projection (RP)}

Sketching schemes in the RP class have properties characterized by Lemma \ref{moments-RP-lem} in the Appendix if $\Pi \in \mathbb{R}^{m \times n}$ is a random matrix with the following properties:
\begin{assumption}\label{u-stat-mean-var}
\begin{itemize}
\item[(i)] 
$\mathbb{E} [ \Pi_{ki}  ] = 0$,
$\mathbb{E} [ \Pi_{ki}^2  ] = m^{-1}$
 for all $k \in [m]$ and all $i \in [n]$,
and
$\max_{(k,i) \in [m] \times [n]} \mathbb{E} [ \Pi_{ki}^4  ] = O (m^{-1})$;
\item[(ii)] 
$\mathbb{E} [   \Pi_{ki}     \Pi_{kj} ]  = 0$
and
$\mathbb{E} [ \Pi_{ki}^2\Pi_{kj}^2 ] = m^{-2}$
for all $k \in [m]$ and all $i \neq j \in [n]$;
\item[(iii)] 
$\mathbb{E} [ \Pi_{ki} \Pi_{kj} \Pi_{\ell p} \Pi_{\ell q} ] = 0$ 
for all $k \neq \ell \in [m]$ and all $i \neq j, p \neq q \in [n]$.
\end{itemize}
\end{assumption}



Under Assumption \ref{u-stat-mean-var}, $T_{n1} = O_p( n^{-1/2})$   is asymptotically negligible, and $T_{n2}$ becomes the leading term for RP.  
As discussed above, the $\Pi$ for the  CS only has one non-zero entry in  each column.
Since $\Pi_{ki} \Pi_{kj} = 0$ for all $k, i \neq j$, it is straightforward to check that the above conditions are satisfied.  For Gaussian random projections, 
\[ \text{GP}: \Pi_{ki}\sim N(0,m^{-1}).\]
Since  all elements of $\Pi$ are i.i.d. with mean zero, variance $m^{-1}$ and the fourth moment $O(m^{-1})$, the conditions are also satisfied.  The
SRHT  has
 \[ \text{SRHT}: \Pi = \sqrt{\frac{n}{m} } S H D,\]
 where $S \in \mathbb{R}^{m \times n}$ is a  uniform sampling matrix with replacement,
 $H  \in \mathbb{R}^{n \times n}$ is a normalized Walsh-Hadamard transform matrix,
 and $D \in \mathbb{R}^{n \times n}$ is a diagonal Rademacher matrix with i.i.d. entries of   $\pm 1$. 
The Appendix shows that the conditions for RP hold for SRHT.
 
 


The following theorem gives the asymptotic  mean squared sketching errors of RP schemes.

\begin{theorem}\label{moments-thm-rp}
 If $\Pi$ is a random matrix satisfying {RP},
then, as $n \rightarrow \infty$,
\begin{itemize}
\item[(i)]
$\mathrm{MSE} \left[ \sqrt{m} \left( U^T \Pi^T \Pi V - U^T V \right) / n \right] \\
 \rightarrow \{ \mathbb{E} ( U_i^2)  \mathbb{E} (V_i^2 ) + [\mathbb{E} ( U_i V_i )]^2 \}.$
\item[(ii)] If, in addition, and $\mathbb E[e_iX_i]=0$ and the columns of $\Pi$ are i.i.d., then $m^{1/2} ( \tilde{\beta}_{OLS} - \hat{\beta}_{OLS} \,)  \rightarrow_d N (0, V_0 ).$
\end{itemize}
\end{theorem}
The limiting MSE of {RP}  is simply the product between
two marginal expectations when $\mathbb{E} ( U_i V_i ) = 0$ (and not the expectation of the product).
 It implies
\[\mathrm{MSE} \left[ \sqrt{m} \, c^T( \tilde{\Upsilon}_n -\hat \Upsilon_n) \right]= \mathbb{E} ( e_i^2 )  \mathbb{E} ( c^T X_i  X_i^T c )\]
and is the reason why  the asymptotic variance for $\tilde\beta_{OLS}$ for RP schemes is $V_0$.

If $e_i^2$ and $(c^T X_i)^2$ are positively (respectively, negatively) correlated,
the limiting MSE of {RP} is smaller (respectively, larger) than that of {RS} and {BS}.
The limiting MSE is the same if $e_i^2$ and $(c^T X_i)^2$ are uncorrelated.


Asymptotic normality of $\tilde\beta$ can be established by applying a central limit theorem for degenerate $U$-statistic  if  the columns of $\Pi$ are i.i.d., as reported in part (ii) of Theorem~\ref{moments-thm-rp}.
The SRHT and SRFT are  not covered by this result because the columns of their $\Pi$ matrix  are not i.i.d. and requires a limit theorem for a particular type of mixing data. In general, establishing asymptotic normality of $\tilde\beta$ based on SRHT or SRFT require different proof techniques. The approach taken in \citet{ahfock-astle-richardson}  
is  to condition on the  data $\mathcal{D}_n$ and apply  a central limit theorem for  a triangular array of random variables.
We do not condition on the data and appeal to the theory of degenerate $U$-statistics. Though  deriving distribution theory for the  SRHT and SRFT estimates  is not straightforward, we will show in simulations  that their finite sample properties are similar to those of CS.

\section{Two-Stage Least Squares}\label{sec:2sls}

The 2SLS estimator  is appropriate when $\mathbb E(X_ie_i) \ne 0$ but exongeous instruments $Z_i$ satisfying $\mathbb E(Z_ie_i)=0$ are available. The 2SLS estimator is 
\[\hat\beta_{2SLS}=(X^TP_Z X)^{-1} X^T P_Z y,\] where $P_Z:=Z(Z^T Z)^{-1}Z^T$ is the projection matrix. 
The estimator
 first projects on $Z$ to  purge  the variations in $X$ correlated with $e$, and in the second step replaces $X$ with  $\hat X=P_Z X$.  
Let  $\hat{\Upsilon}_n := {{Z}}^T   {{e}}/n$ and
$\hat{\Xi}_n := [( X^T  Z/n) ( Z^T Z/n)^{-1} ( Z^T X/n)]^{-1} ( X^T  Z/n) ( Z^T  Z/n)^{-1}$. Analyzing    $\hat\beta_{2SLS}-\beta_0=\hat A_n \hat g_n$ under Assumption~\ref{IV-textbook} given below, as $n \rightarrow \infty$, we have
\begin{align*}
\sqrt{n} (\hat{\beta}_{2SLS} - \beta_0) \rightarrow_d N (0,W_1), 
\end{align*}
where $W_1 := A \,\mathbb{E} (e_i^2 Z_i Z_i^T ) \,A^T$ with
\begin{align*}
A &:= [ \mathbb{E} ( X_i Z_i^T ) [ \mathbb{E} ( Z_i Z_i^T ) ]^{-1} \mathbb{E} ( Z_i X_i^T ) ]^{-1} \\
&\times \mathbb{E} ( X_i Z_i^T ) [ \mathbb{E} ( Z_i Z_i^T ) ]^{-1}.
\end{align*}
Under homoskedasticity, $\mathbb E(e_i^2|Z_i)=\sigma^2$ and $W_1$ reduces to
\[
W_0 := \mathbb{E} ( e_i^2 ) [ \mathbb{E} ( X_i Z_i^T ) [ \mathbb{E} ( Z_i Z_i^T ) ]^{-1} \mathbb{E} ( Z_i X_i^T ) ]^{-1}.
\]
A sketched version of the 2SLS estimator is
\begin{align*}
\tilde\beta_{2SLS}:= (\tilde X^T P_{\tilde Z} \tilde X)^{-1} \tilde X^T P_{\tilde Z} \tilde y.
\end{align*}
 
We now provide  some  algorithmic results not previously documented in the literature.

\begin{assumption}\label{embed}
Let data $\mathcal{D}_n=\{ (y_i,X_i,Z_i) \in \mathbb{R}^{1+p+q}
: i=1,\ldots,n \}$ be fixed, $Z^TZ$ and $X^T P_Z X$ are non-singular, and $Z$ has singular value decomposition $Z=U_Z \Sigma_Z V_Z^T$. 
For  given  constants $\varepsilon_1, \varepsilon_2, \varepsilon_3, \delta \in (0,1/2)$, the following holds jointly with probability at least $1-\delta:$
\begin{align*}
&(i)  \; \left\| U_Z^T \Pi^T \Pi  U_Z -  I_q \right\|_2 \leq  \varepsilon_1, \\
&(ii)  \; \left\| U_Z^T \Pi^T \Pi  U_X -  U_Z^T U_X \right\|_2 \leq  \varepsilon_2, \\
&(iii)  \; \left\| U_Z^T \Pi^T 
\Pi \hat{e} -  U_Z^T \hat{e} \right\| \leq  \varepsilon_3  \left\|  \hat{e} \right\|, \\
&(iv)\; \sigma_{\min}^2 ( U_Z^T U_X ) 
\geq 2 f_1 (\varepsilon_1,  \varepsilon_2),
\end{align*}
where $f_1 (\varepsilon_1,  \varepsilon_2) := [\varepsilon_1 + \varepsilon_2 (\varepsilon_2  + 2)]/({1-\varepsilon_1})$.
\end{assumption}

Low level conditions for Assumption \ref{embed}(i)-(iii) are given in
 \citet{cohen_et_al:ICALP2016}, among others. 
Assumption \ref{embed}(i) is equivalent to the statement that the all eigenvalues of
$U_Z^T \Pi^T \Pi  U_Z$ are bounded between $[1 - \varepsilon_1, 1 + \varepsilon_1]$.
This ensures that 
$\tilde{{Z}}^T \tilde{{Z}}$ is non-singular with probability at least $1-\delta$. Part (iv)  strengthens non-singularity of $X^T P_Z X$ to require that $\sigma_{\min}^2 ( U_Z^T U_X )$ is strictly positive and bounded below by the constant  $2 f_1 (\varepsilon_1,  \varepsilon_2)$.

\begin{theorem}\label{main-thm}
Under  Assumptions \ref{embed},
 the following holds  with probability at least $1-\delta:$
\begin{align*}
\norm{ \tilde{\beta}_{2SLS} - \hat{\beta}_{2SLS} }
&\leq 
\frac{f_2 (\varepsilon_1,  \varepsilon_2) + \varepsilon_3 \left\|  \hat{e} \right\| 
\left[ 1 +
f_2 (\varepsilon_1,  \varepsilon_2)   \right]}{ \sigma_{\min}(X) \sigma_{\min}^2 ( U_Z^T U_X ) } \\
& \times \left[ 1+ \frac{2 f_1 (\varepsilon_1,  \varepsilon_2)}{ \sigma_{\min}^2 ( U_Z^T U_X ) } \right],
\end{align*}
where $f_2 (\varepsilon_1,  \varepsilon_2) :=  \varepsilon_2 + \varepsilon_1/(1-\varepsilon_1)  + \varepsilon_2 \varepsilon_1/(1-\varepsilon_1).$
\end{theorem}
The sketched estimator $\tilde\beta_{2SLS}$  involves, firstly, a regression of $\tilde X$ on $\tilde Z$, and then a regression of $\tilde y$ on the fitted values in the first step. The estimator  thus depends on adequacy of subspace approximation in both steps.
Theorem~\ref{main-thm} provides a worst-case bound for $\tilde\beta_{2SLS}-\hat\beta_{2SLS}$ with the data $\mathcal{D}_n$ being fixed. It depends on
(i)  $\varepsilon_j, j=1,2,3$, 
(ii) variability of $\left\|  \hat{e} \right\|$, 
(iii) the signal from $X$ as given by  $\sigma_{\min}(X)$,  and
(iv) instrument strength as given by $\sigma_{\min} ( U_Z^T U_X )$.
The   sketched  estimator  
can be arbitrarily close to the full sample estimate with high probability, provided that the subsample size $m$ is sufficiently large, $X$ is linearly independent, and the instrument $Z$ is sufficiently relevant for $X$. 

Though  2SLS is a two step estimator, we can still write 
\begin{align*}
\tilde{\beta}_{2SLS} - \hat{\beta}_{2SLS} 
&= (\tilde{\Xi}_n - \hat \Xi_n) \hat\Upsilon_n + \hat\Xi_n ( \tilde{\Upsilon}_n - \hat \Upsilon_n) 
\\ &+ (\tilde{\Xi}_n -\hat \Xi_n) ( \tilde{\Upsilon}_n -\hat \Upsilon_n)
\end{align*}
as in the OLS case, but  now
$\tilde{\Upsilon}_n := \tilde{{Z}}^T   \tilde{{e}}/n$, and
$$
\tilde{\Xi}_n := [(\tilde X^T \tilde Z/n) (\tilde Z^T\tilde Z/n)^{-1} (\tilde Z^T\tilde X/n)]^{-1} (\tilde X^T \tilde Z/n) (\tilde Z^T \tilde Z/n)^{-1}.
$$  A statistical analysis of $\tilde \beta_{2SLS}$ requires additional  assumptions.


\begin{assumption}
\label{IV-textbook}
\begin{itemize}
\item[(i)] The data $\mathcal{D}_n := \{(y_i, X_i, Z_i) \in \mathbb{R}^{1+p+q}
: i=1,\ldots,n \}$ are i.i.d.  with $p \leq q$.  Furthermore,
$X$ and ${Z}$ have  singular value decomposition 
${X} = U_X \Sigma_X V_X^T$
and
${Z} = U_Z \Sigma_Z V_Z^T$.
\item[(ii)] 
 $\mathbb{E}(y_i^4)<\infty$, 
$\mathbb{E}(\| X_i \|^4)<\infty$,
$\mathbb{E}(\| Z_i \|^4)<\infty$,
and
$\mathbb{E} ( X_i X_i^T )$   and $\mathbb{E} ( Z_i X_i^T )$ have full rank $p$.
\item[(iii)] 
The random matrix $\Pi$ is independent of $\mathcal{D}_n$. 
\item[(iv)] 
$m = m_n \rightarrow \infty$ but $m/n \rightarrow 0$ as $n \rightarrow \infty$, while $p$ and $q$ are fixed.
\end{itemize}
\end{assumption}
Arguments similar to those used to prove Theorem~\ref{corr-CLT-thm} lead to the following.

\begin{theorem}\label{main-CLT-thm}
Let Assumption \ref{IV-textbook} hold and  $\mathbb E(Z_ie_i)=0$.   
If $\Pi$ is RP satisfying RP(i)-(iii) with columns that are i.i.d.
  \begin{itemize}
\item[(i)] Under BS, $m^{1/2} ( \tilde{\beta}_{2SLS} - \hat{\beta}_{2SLS} \,)  \rightarrow_d N (0, W_1 )$.
\item[(ii)] Under RP, $m^{1/2} ( \tilde{\beta}_{2SLS} - \hat{\beta}_{2SLS} \,)  \rightarrow_d N (0, W_0 )$.
\end{itemize}
\end{theorem}

Theorem~\ref{main-CLT-thm} provides statistical properties of the sketched 2SLS estimator in Theorem~\ref{main-thm} to complement    the algorithmic results.

Theorem~\ref{main-CLT-thm} states that when the data are sketched by RP,
 $\tilde{\beta}$  is asymptotically normally distributed 
with mean $\hat{\beta}$ and variance  $W_0/m$. Under our assumptions,  $W_0$ can be consistently estimated by
\begin{align*}
\hat{W}_0 := \hat{e}^T \hat{e} \left( {X}^T {Z} ( {Z}^T {Z} )^{-1} {Z}^T {X} \right)^{-1}, 
\end{align*}
where $\hat{e} := {y} - {X} \hat{\beta}$  (not the residuals from the second step).

 Interestingly, the asymptotic variance $W_0$ is the same as if the errors in the full sample regression were homoskedastic. But the result follows from estimation using sketched data rather than by assumption.  This is not the case when inference is based on    the full sample estimates, or the  estimates computed from sketched data  of the  RS type. In such cases,  a homoskedastic covariance weighting matrix would be inefficient since
$\mathbb{E} ( e_i^2 | Z_i ) \neq \mathbb{E} ( e_i^2 )$.

Our analysis can be extended to 
the two-sample 2SLS estimator analyzed in 
 \citet{AK:92,AK:95} and \citet{Inoue:Solon:10}.
However, it is not pursued for brevity of the paper.

In the econometrics literature, the instruments are said to be relevant if $\mathbb E[Z_iX_i^T]\ne 0$. The latter is formalized  by the rank condition in Assumption \ref{IV-textbook}(ii). Tests for instrument relevance usually require robust standard errors corresponding to the parameter estimates in  a regression of $X$ on $Z$  unless heteroskedasticity can be ruled out. An implication of our preceding analysis is that this is not necessary when the  regression is estimated on data sketched by RP, as will be illustrated below.

\section{Practical Inference}

In applications, researchers would like to test a hypothesis about $\beta_0$ using a sketched estimate, and our results provide all the quantities required for inference. 
 In the exogenous regressor case, we generically have
\begin{eqnarray*}
\tilde V_{m}^{-1/2}(\tilde\beta_{OLS}-\beta_0)\approx N(0,I_p) 
\end{eqnarray*}
where the form of $\tilde V_{m}$ depends on $\Pi$. For any $\Pi$ in BS or RP class,
we can use \citet{white1980}'s heteroskedasticity-consistent estimator: 
 \[\tilde V_m=\tilde V_{1,m}= (\tilde X^T\tilde X)^{-1} (\sum_{i=1}^m \tilde X_i\tilde X_i^T \tilde e_i^2) (\tilde X^T\tilde X)^{-1}.\]
For $\Pi$ in the RP class, we can let $\tilde s^2_{OLS}:=\frac{1}{m}\sum_{i=1}^m (\tilde y_i-\tilde X_i^T \tilde\beta_{OLS})^2$. Then
 without assuming homoskedasticity, 
\[\tilde V_m=\tilde V_{0,m}= \tilde s_{OLS}^2(\tilde X^T\tilde X)^{-1},\]

In the endogenous regressor case, $\tilde W_m^{-1/2}(\tilde \beta_{2SLS}-\beta_0)\approx N(0,I_p)$. For RP, we let  $\tilde s^2_{2SLS}:=\frac{1}{m}\sum_{i=1}^m ( \tilde{y}_i - \tilde{X}_i^T \tilde{\beta}_{2SLS})^2$ and
$
\tilde{W}_{0,m} := \tilde s^2_{2SLS} \left( \tilde{X}^T \tilde{Z} ( \tilde{Z}^T \tilde{Z} )^{-1} \tilde{Z}^T \tilde{X} \right)^{-1}.
$ For BS, we define  $\tilde W_{1m}$  from $\tilde W_1$.

Sketching estimators require a choice of $m$. From the algorithmic perspective, $m$ needs to be chosen as small as possible to achieve computational efficiency. 
 If $\Pi$ is constructed from SRHT, the size of $m$ is roughly  (ignoring the log factors) of order $q$ in the best case.  
The requirement for countsketch is more stringent 
 and is proved in the appendix (see Theorem~\ref{main-thm:cs}).
In view of this, we may set 
\begin{align*}
m_1 = C_m q \log q \; \text{ or } \; m_1 = C_m q^2, 
\end{align*}
where $C_m$ is a constant that needs to be chosen by a researcher. 
However, statistical analysis often cares about the variability of the estimates in repeated sampling and  a larger $m$ may be  desirable from the perspective  of statistical efficiency.   An \emph{inference-conscious}  guide $m_2$ can be obtained  in 
 \citet{LN:2020}  by targeting the power at $\bar\gamma$ of a one-sided $t$-test for given nominal size $\bar\alpha$. In particular, let $\tilde\beta$ be $\tilde\beta_{OLS}$ if $\mathbb E[X_i e_i]=0$ and let $\tilde\beta$ be $\tilde \beta_{2SLS}$ when $\mathbb E[X_i e_i]\ne 0$ but $\mathbb E[Z_i e_i]=0$. For
 pre-specified effect size $c^T(\beta^0-\beta_0)$, 
\begin{equation*}
  m_2(m_1)=m_1 S^2(\bar\alpha,\bar\gamma)
  \left[ \frac{\textsc{se}(c^T \tilde\beta)} {c^T(\beta^0-\beta_0)]} \right]^2,
\end{equation*}
where $S(\alpha,\gamma):=\Phi^{-1}(\gamma) +\Phi^{-1}(1-\alpha)$
and $\textsc{se}(c^T \tilde\beta)$ is the standard error of $c^T \tilde\beta$.

Alternatively, 
a data-oblivious  sketch size for a pre-specified  $\tau_2(\infty)$ is defined as
\begin{equation}
  \label{m2rule2}
  m_3=n\frac{S^2(\bar\alpha,\bar\gamma)}{\tau_2^2(\infty)}.
\end{equation}
Note that $m_3$  only requires the choice of $\bar \alpha, \bar\gamma,$ and $\tau_2(\infty)$ which, unlike $m_2$, can be computed without a preliminary sketch.
The condition $m / n \rightarrow 0$ can be viewed as $\tau_2 (\infty) \rightarrow \infty$ as $n \rightarrow \infty$.

\section{Monte Carlo Experiments}\label{sec:mc} 

In this section, we use Monte Carlo experiments 
 to establish that when the errors are homoskedastic, estimates  based on data sketched by random sampling or random projections will yield accurate inference. However,  when the errors are heteroskedastic, sketching by random sampling will yield tests with size distortions, rejecting with much higher probability than the nominal size, unless robust standard errors are used.

\subsection{When All the Regressors are Exogenous}  

We first consider the simulation design for which all the regressors are exogenous.
The regressors $X_i = (1, X_{2,i}, \ldots, X_{p,i})^T$ consist of a constant term and a $(p-1)$-dimensional 
random vector $(X_{2,i}, \ldots, X_{p,i})^T$ generated from a multivariate normal distribution with mean zero vector 
and the variance covariance matrix $\Sigma$, whose $(i,j)$ component is $\Sigma_{ij} = \rho^{|i-j|}$ with $\rho = 0.5$. 
The dependent variable is generated by 
\[
y_i = X_i^T \beta_0 + \sigma (X_i) e_i, 
\]
where
$\beta_0 = (0, 1, \ldots, 1)^T$,
and $e_i$ is generated from $N(0,1)$ independently from $X_i$. 
We consider two designs for $\sigma(X_i)$: (i) homoskedastic design $\sigma(X_i) = 1$ for all $i$
and
(ii) heteroskedastic design $\sigma(X_i) =  \exp ( X_{p,i} )$,
where $X_{p,i}$ is the $p$-th element of $X_i$.
Throughout the Monte Carlo experiment, we set
$n = 10^6$, $m = 500$, and $p = 6$.
There were 5,000 replications for each experiment.
Six sketching methods are considered: 
(i) Bernoulli sampling, 
(ii) uniform sampling, 
(iii) leverage score sampling and reweighted regression as in \citet{ma2020asymptotic};
(iv) countsketch,
(v) SRHT, 
(vi) subsampled randomized Fourier transforms using the real part of 
fast discrete Fourier transform (SRFT).
Table~\ref{tab:size:ols} reports the empirical size and power of the $t$-test.
The null and alternative hypotheses are that $H_0: c^T \beta_0 = 1$  vs. $H_1: c^T \beta_0 \neq 1$ with $c^T = (0,\ldots, 0, 1)$. Equivalently, the null hypothesis is $\beta_p=1$. The power is obtained for the null value $c^T \beta_0  = 1.1$ for the homoskedastic design
and $c^T \beta_0  = 1.4$ for the heteroskedastic design, respectively.
The nominal size is 0.05.

\begin{table}[tb]
\caption{OLS based t test for $H_0: \beta_p=1$ vs $H_1: \beta_p\ne 1$.  S.E.0 and S.E.1 refer to homoskedasticity-only and 
heteroskedasticity-consistent standard errors, respectively.}
\label{tab:size:ols}
\begin{center}
\begin{small}
\begin{sc}
\begin{tabular}{ccccc}
  \toprule
  & (1) & (2) & (3) & (4) \\
  \midrule
  &\multicolumn{2}{c}{Size} & \multicolumn{2}{c}{Power} \\
  &  s.e.0 & s.e.1  & s.e.0 & s.e.1 \\
  \midrule
 \multicolumn{5}{l}{(i) Homoskedastic Design}  \\ 
bernoulli & 0.046 & 0.050 & 0.490 & 0.496 \\ 
  uniform & 0.047 & 0.052 & 0.489 & 0.490 \\ 
  leverage & 0.045 & 0.053 & 0.483 & 0.513 \\ 
  countsketch & 0.049 & 0.051 & 0.479 & 0.489 \\ 
  srht & 0.056 & 0.061 & 0.492 & 0.498 \\ 
  srft & 0.055 & 0.057 & 0.484 & 0.489 \\ 
  \midrule
\multicolumn{5}{l}{(ii) Heteroskedastic Design} \\   
bernoulli & 0.310 & 0.047 & 0.713 & 0.436 \\ 
  uniform & 0.301 & 0.053 & 0.719 & 0.435 \\ 
  leverage & 0.183 & 0.051 & 0.727 & 0.529 \\ 
  countsketch & 0.054 & 0.057 & 0.813 & 0.812 \\ 
  srht & 0.054 & 0.056 & 0.804 & 0.809 \\ 
  srft & 0.050 & 0.052 & 0.799 & 0.806 \\   
\bottomrule
\end{tabular}
\end{sc}
\end{small}
\end{center}
\vskip -0.1in
\end{table}

In column (1) in Table~\ref{tab:size:ols}, we report the size of the test, namely, the probability of rejecting $H_0$ when the null value is true.  
In this column, the $t$-statistic is constructed using homoskedasticity-only standard errors S.E.0. 
Though many methods perform well, both Bernoulli and uniform sampling show substantial size distortions for the heteroskedastic design.
Leverage score sampling combined with reweighted regression seems to account for heteroskedasticity to some extent, but not enough to remove all size distortions. In column (2) which reports results using robust standard errors S.E.1,  all methods have  satisfactory size. 
In column (3), we report the power of the test, i.e., the probability of rejecting $H_0$ when the null value is false.
For the heteroskedastic design,  the powers of the tests using homoskedastic standard errors S.E.0 are worse for Bernoulli, uniform and leverage samplings  than
those for countsketch, SRHT, and SRFT. 
The power loss of the RS schemes  is much more pronounced when the robust standard errors are used in column (4).
This efficiency loss is consistent with asymptotic theory developed in the paper because
the squared regression error is positively correlated with one of the elements of squared $X_i$ under the heteroskedastic design. 
All RP schemes  perform similarly, hinting that even though a formal proof awaits future research, asymptotic normality may also hold for both 
SRHT and SRFT, and  not just countsketch.

\subsection{When One of the Regressors is Endogenous}  

We now move to the case when  the regressors are $X_i=(1,X_{2,i}, \ldots, X_{p-1,i}, X_{p,i})^T$, and $y_i$ is generated by 
\begin{align}\label{outcome-eq}
y_i = X_i^T \beta_0 + \sigma_2 (Z_i) (\eta_i + \epsilon_i), 
\end{align}
where  $\epsilon_i\sim N(0,1)$ is randomly drawn   independently from $X_i$ and $\eta_i$. The first $p-1$ regressors, including the intercept term, are exogenous, but
\begin{align}\label{1st-reg}
X_{p,i} = Z_i^T \zeta_0  + \sigma_1 (Z_i) \eta_{i}, 
\end{align}
where $\zeta_0=(\zeta_{1,0}, \ldots, \zeta_{q,0})^T$, $\eta_{i}\sim N(0,1)$ independently from $Z_i=(1, Z_{2,i}, \ldots, Z_{q,i})^T$.
The presence of $\eta_i$ in both \eqref{outcome-eq} and \eqref{1st-reg}  induces endogeneity of $X_{p,i}$. 

In each of the 1000 replications,  $(X_{2,i},\ldots, X_{p-1,i})^T=(Z_{2,i},\ldots, Z_{p-1,i})^T$, while
 the $(q-1)$-dimensional  $Z_i$ is  multivariate normal with mean zero
and the variance  $\Sigma$, whose $(i,j)$ component is $\Sigma_{ij} = \rho^{|i-j|}$ with $\rho = 0.5$.  We consider two designs for $\sigma_1(Z_i)$: (i) homoskedastic design $\sigma_1(Z_i) = 1$ for all $i$
and
(ii) heteroskedastic design $\sigma_1(Z_i) =  \exp \left( \frac{5}{q} \sum_{j=2}^{q}  |Z_{j,i}| \right) /100$.
As in the previous section, we set
$n = 10^6$, $m = 500$,  $p = 6$, and $q=21$. We consider five sketching schemes and no longer include leverage score sampling  since it is unclear how to implement it in the case of 2SLS.
The nominal size is 0.05.  Throughout, $(\zeta_{1,0},\ldots,\zeta_{p-1,0})= (0,0.1,\ldots,0.1)^T$, but values of $\zeta_{j,0}$ for $j \geq p$ depend on the context as explained below.

We first examine the so-called first-stage F-test for instrument relevance. In this case of   a  scalar endogenous  regressor,   the null hypothesis of irrelevant instruments amounts to a joint test of
$H_0: \zeta_{j,0} =0$ for every $j=p,\ldots, q$ in  \eqref{1st-reg}.
The size of the test is evaluated at
$\zeta_{j,0}= 0$ and   the power  at
$\zeta_{j,0}=0.1$ for $j=p,\ldots, q$. 
The  F-test statistic is constructed as
$$
F = \frac{1}{q-p+1}  \hat{\zeta}_{-(p-1)}^T \left( [\hat V]_{-(p-1),-(p-1)} \right)^{-1} \hat{\zeta}_{-(p-1)},
$$
where
$\hat{\zeta}_{-(p-1)}$ is a $(q-p+1)$-dimensional vector of the OLS estimate $\hat{\zeta}$ 
of regressing $X_{p,i}$ on  $Z_i$, excluding
the first $(p-1)$ elements, and $[\hat V]_{-(p-1),-(p-1)}$ is the corresponding submatrix of $\hat V$.

  In Table~\ref{tab:size:F-test}, we report the size and power of the F-test for 
$H_0: \zeta_{p,0} = \zeta_{p+1,0} = \ldots = \zeta_{q,0} = 0$ using
homoskedasticity-only (V.0) and 
heteroskedasticity-consistent (V.1) asymptotic variance estimates, respectively.
As in the previous subsection, Bernoulli and uniform sampling sketches suffer from size distortions in the heteroskedastic design but  V.0 is used. Tests based on V.1 have good size without sacrificing much power when the $F$ test is constructed from data sketched by RP.

\begin{table}[tb]
\caption{ F test for Instrument Relevance: V.0 and V.1 refer to homoskedasticity-only and 
heteroskedasticity-consistent asymptotic variance estimates, respectively.}
\label{tab:size:F-test}
\begin{center}
\begin{small}
\begin{sc}
\begin{tabular}{ccccc}
  \toprule
  & (1) & (2) & (3) & (4) \\
  \midrule
  &\multicolumn{2}{c}{Size} & \multicolumn{2}{c}{Power} \\
  &  V.0 & V.1  & V.0 & V.1 \\
  \midrule
 \multicolumn{5}{l}{(i) Homoskedastic Design}  \\ 
bernoulli & 0.047 & 0.063 & 1.000 & 0.999 \\ 
  uniform & 0.049 & 0.063 & 0.997 & 0.999 \\ 
  countsketch & 0.040 & 0.058 & 1.000 & 0.999 \\ 
  srht & 0.048 & 0.051 & 0.999 & 0.998 \\ 
  srft & 0.050 & 0.052 & 1.000 & 0.999 \\ 
  \midrule
\multicolumn{5}{l}{(ii) Heteroskedastic Design} \\   
bernoulli & 0.350 & 0.033 & 0.914 & 0.843 \\ 
  uniform & 0.338 & 0.024 & 0.900 & 0.828 \\ 
  countsketch & 0.045 & 0.060 & 0.879 & 0.883 \\ 
  srht & 0.038 & 0.052 & 0.897 & 0.895 \\ 
  srft & 0.050 & 0.059 & 0.890 & 0.888 \\   
\bottomrule
\end{tabular}
\end{sc}
\end{small}
\end{center}
\vskip -0.1in
\end{table}

\begin{table}[tb]
\caption{2SLS based  t test for $H_0: \beta_p=1, H_1: \beta_p\ne 1$}
\label{tab:size:2sls}
\begin{center}
\begin{small}
\begin{sc}
\begin{tabular}{ccccc}
  \toprule
  & (1) & (2) & (3) & (4) \\
  \midrule
  &\multicolumn{2}{c}{Size} & \multicolumn{2}{c}{Power} \\
  &  s.e.0 & s.e.1  & s.e.0 & s.e.1 \\
  \midrule
 \multicolumn{5}{l}{(i) Homoskedastic Design}  \\ 
bernoulli & 0.065 & 0.067 & 0.687 & 0.695 \\ 
  uniform & 0.056 & 0.057 & 0.686 & 0.693 \\ 
  countsketch & 0.055 & 0.060 & 0.698 & 0.705 \\ 
  srht & 0.043 & 0.046 & 0.710 & 0.714 \\ 
  fft & 0.061 & 0.068 & 0.704 & 0.703 \\ 
  \midrule
\multicolumn{5}{l}{(ii) Heteroskedastic Design} \\   
bernoulli & 0.274 & 0.050 & 0.844 & 0.648 \\ 
  unif & 0.291 & 0.046 & 0.864 & 0.654 \\ 
  countsketch & 0.042 & 0.047 & 0.930 & 0.930 \\ 
  srht & 0.052 & 0.056 & 0.941 & 0.944 \\ 
  fft & 0.055 & 0.055 & 0.933 & 0.942 \\ 
\bottomrule
\end{tabular}
\end{sc}
\end{small}
\end{center}
\vskip -0.1in
\end{table}

We now turn  to 2SLS estimation of $\beta_0$. To ensure that the instruments are powerful enough to estimate $\beta_0$ well, we  now set $\zeta_{j,0}=0.5$ for $j=p,\ldots, q$ with
$\sigma_1 (Z_i) = 1$ for all $i$. We set $\beta_0 = (0, 1, \ldots, 1)^T$
and  consider two designs for $\sigma_2(Z_i)$: (i) homoskedastic design $\sigma_2(Z_i) = 1$ for all $i$
and
(ii) heteroskedastic design $\sigma_2(Z_i) =  \exp \left( \frac{5}{q} \sum_{j=2}^{q}  |Z_{j,i}| \right) /100$.

As in the previous subsection, we test  $H_0: \beta_p=1$ against $H_1: \beta_p\ne 1$, or equivalently,  $c^T = (0,\ldots, 0, 1)$. The power is obtained for $ \beta_p  = 1.05$ in the homoskedastic design
and $ \beta_p  = 1.10$ for the heteroskedastic design, respectively.
Table~\ref{tab:size:2sls} reports results for 
 nominal size of 0.05.  
Basically,  the same patterns are observed as in the previous section.  Thus, simulations support the theoretical result  that  robust standard errors are not needed for  inference  when estimation is based on sketched data using sketching schemes  in the RP class. 

\section{An Empirical Illustration}\label{sec:example} 

An exemplary application of the 2SLS in economics is causal inference, such as to estimate the return to education. Suppose that 
$y_i$ is the wages for worker $i$ (typically in logs)
and
$X_i$ contains educational attainment $\texttt{edu}_i$ (say, years of schooling completed).
Here, the unobserved random variable $e_i$ includes worker $i$'s unobserved ability among other things.
Then, $\texttt{edu}_i$ will be correlated with $e_i$ if workers with higher ability tends to attain higher levels of education.
The least-squares estimator may not provide  a consistent estimate of the return to schooling.
To overcome this problem, economists use an instrumental variable that is uncorrelated with $e_i$ but correlated with $\texttt{edu}_i$.
We re-examine the OLS and 2SLS estimates of  return to education  in columns (1) and (2) of Table IV in \citet{AK1991}. The dependent variable $y$ is the log weekly wages,
the covariates $X$ include years of education, the intercept term and  9 year-of-birth dummies $(p=11)$. 
Following \citet{AK1991}, the instruments $Z$ are the exogenous regressors (i.e., the intercept and year-of-birth dummies) and
a full set of quarter-of-birth (one quarter omitted) times year-of-birth interactions $(q= 1 + 9 + 3 \times 10 = 40)$. Their idea was that season of birth is unlikely to be correlated with workers' ability but
can affect educational attainment because of compulsory schooling laws.
The full sample size is $n = 247,199$.

\begin{table}[tb]
\caption{OLS in the empirical illustration: S.E.0 and S.E.1 refer to homoskedasticity-only and 
heteroskedasticity-consistent standard errors, respectively ($n = 247,199$, $m = 15,283$)}
\label{tab:ols:example}
\begin{center}
\begin{small}
\begin{sc}
\begin{tabular}{cccc}
  \toprule
  & estimate & s.e.0  &  s.e.1 \\ 
  \midrule
full sample & 0.08016 & 0.00036 & 0.00039 \\ 
\midrule
 bernoulli & 0.07989 & 0.00142 & 0.00158 \\ 
 uniform & 0.07931 & 0.00146 & 0.00163 \\ 
  leverage & 0.07779 & 0.00144 & 0.00149 \\ 
  countsketch & 0.08105 & 0.00143 & 0.00147 \\ 
  srht & 0.07975 & 0.00142 & 0.00143 \\ 
  srft & 0.08296 & 0.00143 & 0.00143 \\ 
\bottomrule
\end{tabular}
\end{sc}
\end{small}
\end{center}
\vskip -0.1in
\end{table}

\begin{table}[tb]
\caption{2SLS in the empirical illustration ($n = 247,199$, $m = 61,132$)}
\label{tab:2sls:example}
\begin{center}
\begin{small}
\begin{sc}
\begin{tabular}{cccc}
  \toprule
  & estimate & s.e.0  &  s.e.1 \\ 
  \midrule
full sample & 0.077 & 0.015 & 0.015 \\ 
\midrule
  bernoulli & 0.053 & 0.027 & 0.028 \\ 
  uniform & 0.094 & 0.021 & 0.021 \\ 
  countsketch & 0.076 & 0.021 & 0.023 \\ 
  srht & 0.115 & 0.018 & 0.018 \\ 
  srft & 0.081 & 0.022 & 0.022 \\ 
\bottomrule
\end{tabular}
\end{sc}
\end{small}
\end{center}
\vskip -0.1in
\end{table}

To construct sketched data, we need to choose the sketch size $m$.
We use a  data-oblivious  sketch size $m_3$ defined in \eqref{m2rule2}   with target size set to $\alpha= 0.05$  and  target power to $\gamma = 0.8$, giving $S^2(\bar\alpha,\bar\gamma) = 6.18$. It remains to specify $\tau_2(\infty)$, which can be interpreted as
the value of $t$-statistic when the sample size is really large. 

In the OLS case, we take $\tau_2(\infty) = 10$  resulting in 
$m = 15,283$ (about 6\% of $n$).
Table~\ref{tab:ols:example} reports empirical results for the OLS estimates.
For each sketching scheme, only one random sketch is drawn; hence, the results can change if we redraw sketches.
Remarkably, all sketched estimates are 0.08, reproducing the full sample estimate up to the second digit.
The sketched homoskedasticity-only standard errors are also very much the same across different methods.
The Eicker-Huber-White standard error S.E.1 is a bit larger than the homoskedastic standard error S.E.0 with the full sample. As expected, the same pattern is observed for Bernoulli and uniform sampling, as these sampling schemes preserve conditional heteroskedasticity.

For 2SLS, as it is more demanding to achieve good precision, we take $\tau_2(\infty) = 5$,  resulting in 
$m = 61,132$ (about 25\% of $n$).
Table~\ref{tab:2sls:example} reports empirical results for the 2SLS estimates.
The sketched estimates vary from 0.053 to 0.115, reflecting that the 2SLS estimates are less precisely estimated than the OLS estimates. Both types of standard errors are almost identical across all sketches for 2SLS, suggesting that heteroskedasticity is not an issue in this data.

\newpage
\appendix

\section{Appendix: Proofs for OLS}\label{sec:proofs_ols}

Recall that $\hat\beta_{OLS}-\beta_0=(X^TX)^{-1}X^T e$ and $\tilde\beta_{OLS}-\beta_0=(\tilde X^T\tilde X)^{-1}\tilde X^T\tilde e$. Thus
\begin{align*}
\tilde{\beta}_{OLS} - \hat{\beta}_{OLS} 
&=\bigg((\tilde X^T\tilde X)^{-1}-(X^T X)^{-1}\bigg) X^Te+(\tilde X^T \tilde X)^{-1}\bigg(\tilde X^T\tilde e-X^T e\bigg)\\ &+\bigg((\tilde X^T \tilde X)^{-1}-(X^T X)^{-1}\bigg)\bigg(\tilde X^T\tilde e-X^T e\bigg)\\
&= (\tilde{\Xi}_n - \hat{\Xi}_n) \hat\Upsilon_n + \hat{\Xi}_n ( \tilde{\Upsilon}_n - \hat \Upsilon_n) 
+ (\tilde{\Xi}_n - \hat{\Xi}_n) ( \tilde{\Upsilon}_n -\hat \Upsilon_n), 
\end{align*}
where 
$\tilde{\Upsilon}_n := \tilde{{X}}^T   \tilde{{e}}/n$,
$\hat \Upsilon_n := {X}^T {e}/n$,
$\tilde{\Xi}_n := (\tilde X^T\tilde X/n)^{-1}$, and
$\hat{\Xi}_n :=  ( X^T X/n)^{-1}.$ 
By the law of large numbers and the continuous mapping theorem,
$\hat \Xi_n - \Xi = o_p (1)$  
and by the central limit theorem, $\hat g_n=O_p(n^{-1/2})$. 
Furthermore, by repeated applications of Theorem~\ref{moments-thm},
$$
\textrm{MSE} [(\tilde X^T\tilde e  - X^T e)/n] = O(m^{-1})
\ \ \text{ and } \ \
\textrm{MSE} [(\tilde X^T\tilde X  - X^T X)/n] = O(m^{-1}),
$$
and by Chebyshev's inequality, 
$(\tilde X^T\tilde e  - X^T e)/n = O_p (m^{-1/2})$
and
$(\tilde X^T\tilde X  - X^T X)/n = O_p (m^{-1/2})$.
The latter combined with the continuous mapping theorem yields that
$\tilde{\Xi}_n - \hat{\Xi}_n  = O_p ( m^{-1/2} )$. 
Thus, 
\begin{align*}
\tilde{\beta}_{OLS} - \hat{\beta}_{OLS} 
&= \Xi ( \tilde{\Upsilon}_n - \hat \Upsilon_n)
+ 
 (\hat{\Xi}_n - \Xi) ( \tilde{\Upsilon}_n - \hat \Upsilon_n) +
O_p ( m^{-1/2} n^{-1/2}  + m^{-1} ) \\
&= 
\Xi ( \tilde{\Upsilon}_n - \hat \Upsilon_n)
+ 
o_p ( m^{-1/2}  ).
\end{align*}
We start with asymptotic normality for Bernoulli sampling. 

\begin{proof}[Proof of Theorem~\ref{corr-CLT-thm}(i)]
In view of the Cramer-Wold device, it suffices to show that for any nonzero constant vector $c  \in \mathbb{R}^p$, 
\begin{align*}
m^{1/2} \left[ c^T \mathbb{E} ( e_{i}^2 X_i X_i^T ) c \right]^{-1/2}
c^T ( \tilde{\Upsilon}_n - \hat \Upsilon_n)  \rightarrow_d N (0, 1).
\end{align*}
Write 
\begin{align*}
c^T ( \tilde{\Upsilon}_n - \hat \Upsilon_n) 
=
n^{-1} \sum_{i=1}^n    \left( \frac{n}{m} B_{ii} - 1 \right) e_i    X_i^T c.
\end{align*}
Because the summands are i.i.d. with mean zero and finite variance, the central limit theorem yields the desired result immediately.
 \end{proof}
 
\begin{proof}[Proof of Theorem~\ref{corr-CLT-thm}(ii)]
This result is a special case of Theorem~\ref{moments-thm-rp}(ii) and we prove Theorem~\ref{moments-thm-rp}(ii) below.
 \end{proof} 

In what follows, we focus on the instance that $\textrm{r.dim}(\Pi)  = m$.  
Recall that 
\begin{align*}
n^{-1} \left( U^T \Pi^T \Pi V - U^T V \right) 
= n^{-1} \sum_{i=1}^n    \psi_i  U_i    V_i
+ n^{-1} \sum_{i=1}^n  \sum_{j=1, j \neq i}^n  U_i   \varphi_{ij} V_j
=: T_{n1} + T_{n2},
\end{align*}
where
$U \in \mathbb{R}^n$ and $V \in \mathbb{R}^n$ are  vectors of certain
i.i.d. random variables $(U_i, V_i) \in \mathbb{R}^2$ that are independent of $\Pi$,
\begin{align*}
\psi_i := \sum_{k=1}^{m} \Pi_{ki}^2  - 1,
\; 
\varphi_{ij} := \sum_{k=1}^{m} \Pi_{ki}  \Pi_{kj}.
\end{align*}

There are two important cases. In case 
(i), $T_{n1}$ is the leading term and $T_{n2}$ is identically zero.
The latter is true if $\varphi_{ij} = 0$ for all $i \neq j$. Methods in this class generate sketches   by sampling from the full data matrix using deterministic or data dependent  probabilities. The case includes random sampling with replacement (RS). 
 
In case (ii), $T_{n2}$ is the leading term and $T_{n1}$ is identically zero or asymptotically negligible.
Recall that for $w = (u,v, \pi_1, \ldots, \pi_m)^T$,
\[
\tilde{H} (w_1, w_2) : = 
\sum_{k=1}^{m} u_1 \pi_{k1}  \pi_{k2} v_2
\]
and
$
H (w_1, w_2) : =  \tilde{H} (w_1, w_2) + \tilde{H} (w_2, w_1).
$
Then,
\begin{align}\label{Tn2-ustat}
T_{n2}
&= n^{-1} \operatorname*{\sum \sum}_{1 \leq i < j \leq n}  H (W_i, W_j).
\end{align}
Note that $H (W_i, W_j)$ is symmetric, i.e., $H (W_i, W_j) = H (W_j, W_i)$.
The canonical form of $\Pi$ we consider for case (ii) is random projection whose properties are given in the main text. 

Before proving Theorems~\ref{moments-thm} and~\ref{moments-thm-rp}, we first prove Lemma~\ref{RP-example-lem}
and establish some useful lemmas.

\begin{lemma}\label{RP-example-lem}
{GP}, {CS}, and {SRHT} satisfy the conditions for {RP}.
\end{lemma}

\begin{proof}[Proof of Lemma~\ref{RP-example-lem}]
For {GP}, it is straightforward to check all the conditions as 
all elements of $\Pi$ are i.i.d. We omit the details.
For {CS}, note that the columns of $\Pi$ are i.i.d. and $\Pi$ has only one non-zero entry in  each column, hence implying that $\Pi_{ki}     \Pi_{kj} = 0$ for all $k, i \neq j$.
Then, it is easy to see that all the conditions are satisfied.
It is more involving to check the conditions for {SRHT}.
To do so, write
\begin{align*}
\Pi_{ki} = \sqrt{\frac{n}{m} } \sum_{j = 1}^n S_{k j} H_{j i} D_{ii}.
\end{align*}
Using the fact that for each $k$, $S_{k \ell_1} S_{k \ell_2} = 0$ whenever $\ell_1 \neq \ell_2$ (the property of uniform sampling), further write
\begin{align*}
\Pi_{ki} \Pi_{kj} 
&= \frac{n}{m} \sum_{\ell_1 = 1}^n  \sum_{\ell_2 = 1}^n  S_{k \ell_1} H_{\ell_1 i} D_{ii}  S_{k \ell_2} H_{\ell_2 j} D_{jj} \\
&=  \frac{n}{m} \sum_{\ell = 1}^n   S_{k \ell} H_{\ell i} D_{ii}  H_{\ell j} D_{jj}
\end{align*}
and
\begin{align*}
\Pi_{ki} \Pi_{kj}  \Pi_{\ell p} \Pi_{\ell q}  
&=  \frac{n^2}{m^2} \sum_{t_1 = 1}^n   \sum_{t_2 = 1}^n
S_{k t_1} H_{t_1 i}  D_{ii} H_{t_1 j} D_{jj} 
S_{\ell t_2} H_{t_2 p}  D_{pp} H_{t_2 q} D_{qq}.
\end{align*}
Using the facts that
$\mathbb{E} (S_{k j}) = n^{-1}$,
$\mathbb{E} (D_{ii}) = 0$,
$\sum_{j = 1}^n   H_{j i}^2 = 1$,
and
$| H_{j i} | = n^{-1/2}$, 
we have 
\begin{align*}
\mathbb{E} (\Pi_{ki} ) &= \sqrt{\frac{n}{m} } \sum_{j = 1}^n \mathbb{E} (S_{k j}) H_{j i} \mathbb{E} (D_{ii}) = 0, \\
\mathbb{E} (\Pi_{ki}^2) 
&=  \frac{n}{m} \sum_{j = 1}^n  \mathbb{E} (S_{k j}) H_{j i}^2   
=  \frac{1}{m} \sum_{j = 1}^n   H_{j i}^2
= \frac{1}{m}, \\
\mathbb{E} ( \Pi_{ki}^2 \Pi_{kj}^2 )
&=  \frac{n^2}{m^2} \sum_{\ell = 1}^n   \mathbb{E} ( S_{k \ell}) H_{\ell i}^2  H_{\ell j}^2 
= \frac{n}{m^2} \sum_{\ell = 1}^n    H_{\ell i}^2  H_{\ell j}^2
= \frac{1}{m^2}, \\
\mathbb{E} ( \Pi_{ki}^4 )
&=  \frac{n^2}{m^2} \sum_{\ell = 1}^n   \mathbb{E} ( S_{k \ell}) H_{\ell i}^4 
= \frac{n}{m^2} \sum_{\ell = 1}^n    H_{\ell i}^4  
= \frac{1}{m^2}. 
\end{align*}
Furthermore, note that 
the diagonal elements of $D$ are i.i.d. and 
the rows of $S$ are i.i.d. Then, we have that
$\mathbb{E} [   \Pi_{ki}     \Pi_{kj} ]  = 0$
for all $k, i \neq j$
and
$\mathbb{E} [ \Pi_{ki} \Pi_{kj} \Pi_{\ell p} \Pi_{\ell q} ] = 0$ 
for all $k \neq \ell, i \neq j, p \neq q$.
Therefore, we have verified all the required conditions. 
\end{proof}

\begin{lemma}\label{moments-unif-lem}
If $\Pi$ is a random matrix satisfying {RS}, 
then,
\begin{align*}
\mathbb{E} \left[ n^{-1} \left(  U^T \Pi^T \Pi V - U^T V \right) \right] &= 0, \\
\mathrm{Var} \left[ n^{-1} \left(  U^T \Pi^T \Pi V - U^T V \right) \right] &= \left\{ \frac{1}{m} - \frac{1}{n} + \left( 1 - \frac{1}{m} \right)  \sum_{i=1}^n p_i^2 \right\} \mathrm{Var}( U_i V_i ).
\end{align*}
\end{lemma}

In particular, when $p_i = n^{-1}$, the variance is reduced to 
$\frac{n-1}{n} \frac{1}{m} \mathrm{Var}( U_i V_i )$.

\begin{proof}[Proof of Lemma~\ref{moments-unif-lem}]
Using the property of {RS}, we have that 
\begin{align*}
\mathbb{E} [   \psi_i  ] &= \sum_{k=1}^{m} \mathbb{E} [ \Pi_{ki}^2 ]  - 1   =  n p_i - 1, \\
\mathbb{E} [   \psi_i^2 ] 
&= 
\sum_{k=1}^{m} \sum_{\ell=1}^{m} \mathbb{E} [ \Pi_{ki}^2 \Pi_{\ell i}^2 ]
- 2
\sum_{k=1}^{m} \mathbb{E} [ \Pi_{ki}^2 ]  + 1   \\
&=
\sum_{k=1}^{m}  \mathbb{E} [ \Pi_{ki}^4 ]
+
\sum_{k=1}^{m} \sum_{\ell=1, \ell \neq k}^{m} \mathbb{E} [ \Pi_{ki}^2 \Pi_{\ell i}^2 ]
- 2
\sum_{k=1}^{m} \mathbb{E} [ \Pi_{ki}^2 ]  + 1   \\
&=
\frac{n^2}{m} p_i + \frac{n^2 m(m-1)}{m^2}  p_i^2 - 2 n p_i + 1,
\end{align*}
and for $i \neq j$, using the fact that $\Pi_{ki} \Pi_{kj} = 0$ whenever $i \neq j$,
\begin{align*}
\mathbb{E} [   \psi_i \psi_j  ]
&=
\sum_{k=1}^{m} \sum_{\ell=1}^{m} \mathbb{E} [ \Pi_{ki}^2 \Pi_{\ell j}^2 ]
- \sum_{k=1}^{m} \mathbb{E} [ \Pi_{ki}^2 ]  
- \sum_{\ell =1}^{m} \mathbb{E} [ \Pi_{\ell j}^2 ] 
+ 1   \\
&=
\sum_{k=1}^{m}  \mathbb{E} [ \Pi_{ki}^2 \Pi_{kj}^2 ]
+
\sum_{k=1}^{m} \sum_{\ell=1, \ell \neq k}^{m} \mathbb{E} [ \Pi_{ki}^2 \Pi_{\ell j}^2 ]
- \sum_{k=1}^{m} \mathbb{E} [ \Pi_{ki}^2 ]  
- \sum_{\ell =1}^{m} \mathbb{E} [ \Pi_{\ell j}^2 ]   + 1   \\
&=
 \frac{n^2 m(m-1)}{m^2}  p_i p_j - n p_i - n p_j + 1.
\end{align*}
Note that $T_{n2} = 0$ because $\varphi_{ij} = 0$. Hence, it suffices to compute the mean and variance
of $T_{n1}$. Write
\begin{align*}
\mathbb{E} ( T_{n1} )
&= n^{-1} \sum_{i=1}^n   \mathbb{E} ( \psi_i )  \mathbb{E} ( U_i    V_i ) 
= n^{-1} \sum_{i=1}^n   (n p_i - 1)  \mathbb{E} ( U_i    V_i ) 
= 0, \\
\mathrm{Var} ( T_{n1} )
&= n^{-2} \sum_{i=1}^n   \mathbb{E} ( \psi_i^2 )  \mathbb{E} ( U_i^2    V_i^2 ) 
+ n^{-2} \sum_{i=1}^n \sum_{j=1, j \neq i}^n  \mathbb{E} ( \psi_i \psi_j )  \mathbb{E} ( U_i    V_i ) \mathbb{E} ( U_j V_j) \\
&= n^{-2} \sum_{i=1}^n   \left( \frac{n^2}{m} p_i + \frac{n^2 m(m-1)}{m^2}  p_i^2 - 2 n p_i + 1 \right)  \mathbb{E} ( U_i^2    V_i^2 ) \\
&+ n^{-2} \sum_{i=1}^n \sum_{j=1, j \neq i}^n  \left( \frac{n^2 m(m-1)}{m^2}  p_i p_j - n p_i - n p_j + 1 \right)  \mathbb{E} ( U_i    V_i ) \mathbb{E} ( U_j V_j) \\
&=    \left\{ \frac{1}{m} + \left( 1 - \frac{1}{m} \right)  \sum_{i=1}^n p_i^2 - \frac{1}{n}  \right\}  \mathbb{E} ( U_i^2    V_i^2 ) \\
&+    \left\{ \left( 1 - \frac{1}{m} \right)  \left( 1 - \sum_{i=1}^n p_i^2 \right) - \frac{n-1}{n}   \right\}  \mathbb{E} ( U_i    V_i ) \mathbb{E} ( U_j V_j) \\
&=  \left\{ \frac{1}{m} - \frac{1}{n} + \left( 1 - \frac{1}{m} \right)  \sum_{i=1}^n p_i^2 \right\}
\left\{ \mathbb{E} ( U_i^2    V_i^2 ) - \left[ \mathbb{E} ( U_i    V_i ) \right]^2 \right\}.
\end{align*}
Therefore, we have proved the lemma.
\end{proof}

\begin{lemma}\label{moments-RP-lem}
If $\Pi$ is a random matrix satisfying {RP}, 
then,
\begin{align*}
\mathbb{E} \left( T_{n1} \right) = \mathbb{E} \left( T_{n2} \right) = 0, 
\mathrm{Var} \left( T_{n1} \right) = O (n^{-1}), 
\ \ \text{ and } \ \
\mathrm{Var} \left( T_{n2} \right) = \frac{n-1}{n} \frac{1}{m} \{ \mathbb{E} ( U_i^2)  \mathbb{E} (V_i^2 ) + [\mathbb{E} ( U_i V_i )]^2 \}.
\end{align*}
\end{lemma}

\begin{proof}[Proof of Lemma~\ref{moments-RP-lem}]
As in the proof of Lemma~\ref{moments-unif-lem}, we have that 
\begin{align*}
\mathbb{E} [   \psi_i  ] &= \sum_{k=1}^{m} \mathbb{E} [ \Pi_{ki}^2 ]  - 1   = 0, \\
\mathbb{E} [   \psi_i^2 ] 
&=
\sum_{k=1}^{m}  \mathbb{E} [ \Pi_{ki}^4 ]
+
\sum_{k=1}^{m} \sum_{\ell=1, \ell \neq k}^{m} \mathbb{E} [ \Pi_{ki}^2 \Pi_{\ell i}^2 ]
- 2
\sum_{k=1}^{m} \mathbb{E} [ \Pi_{ki}^2 ]  + 1   \\
&= O(1) \ \ \ \text{using the assumption that $\mathbb{E} [ \Pi_{ki}^4  ] = O (m^{-1})$},
\end{align*}
and for $i \neq j$, 
\begin{align*}
\mathbb{E} [   \psi_i \psi_j  ]
&=
\sum_{k=1}^{m}  \mathbb{E} [ \Pi_{ki}^2 \Pi_{kj}^2 ]
+
\sum_{k=1}^{m} \sum_{\ell=1, \ell \neq k}^{m} \mathbb{E} [ \Pi_{ki}^2 \Pi_{\ell j}^2 ]
- \sum_{k=1}^{m} \mathbb{E} [ \Pi_{ki}^2 ]  
- \sum_{\ell =1}^{m} \mathbb{E} [ \Pi_{\ell j}^2 ]   + 1   \\
&=
\frac{1}{m} + \frac{m(m-1)}{m^2} - 2 + 1 \\
&= 0.
\end{align*}
Now write
\begin{align*}
\mathbb{E} ( T_{n1} )
&= n^{-1} \sum_{i=1}^n   \mathbb{E} ( \psi_i )  \mathbb{E} ( U_i    V_i ) = 0, \\
\mathrm{Var} ( T_{n1} )
&= n^{-2} \sum_{i=1}^n   \mathbb{E} ( \psi_i^2 )  \mathbb{E} ( U_i^2    V_i^2 ) 
+ n^{-2} \sum_{i=1}^n \sum_{j=1, j \neq i}^n  \mathbb{E} ( \psi_i \psi_j )  \mathbb{E} ( U_i    V_i ) \mathbb{E} ( U_j V_j) \\
&=  O( n^{-1}).
\end{align*}
Furthermore,  
$\mathbb{E} [ H (W_i, W_j)  ] = 
\mathbb{E} [ \mathbb{E} [ H (W_i, W_j) | U_i, U_j, V_i, V_j ] ] = 0$
by {RP}(ii), specifically, $\mathbb{E} [   \Pi_{ki}     \Pi_{kj} ]  = 0$. 
Hence, $\mathbb{E} ( T_{n2} ) = 0$.

For the final result, 
write
\begin{align*}
\mathbb{E} [ \{ H (W_i, W_j) \}^2 ] 
&= 
  \mathbb{E}
\left\{
\left[ \tilde{H} (W_i, W_j) + \tilde{H} (W_j, W_i) \right]
\left[ \tilde{H} (W_i, W_j) + \tilde{H} (W_j, W_i) \right]
\right\}.
\end{align*}
Write
\begin{align}\label{H-notation-eq}
 \{ H (W_i, W_j) \}^2 = T_{ij1} + T_{ij2} + 2 T_{ij3},
 \end{align}
where
\begin{align*}
T_{ij1} := \tilde{H} (W_i, W_j) \tilde{H} (W_i, W_j), \\
T_{ij2} := \tilde{H} (W_j, W_i) \tilde{H} (W_j, W_i), \\
T_{ij3} := \tilde{H} (W_i, W_j) \tilde{H} (W_j, W_i).
\end{align*}
Use {RP}(ii)-(iii),
in particular, 
$\mathbb{E} [ \Pi_{ki}^2\Pi_{kj}^2 ] = m^{-2}$
and
$\mathbb{E} [ \Pi_{ki} \Pi_{kj} \Pi_{\ell i} \Pi_{\ell j} ] = 0$ whenever $k \neq \ell$, $i \neq j$,
 to obtain
\begin{align*}
\mathbb{E} [ T_{ij1} | U_i, V_i, U_j, V_j] &=
\sum_{k=1}^{m} 
U_i^2 V_j^2 
\mathbb{E} [ \Pi_{ki}^2\Pi_{kj}^2 ] 
 = \frac{1}{m} U_i^2 V_j^2, \\
\mathbb{E} [ T_{ij2} | U_i, V_i, U_j, V_j] 
&=
\sum_{k=1}^{m} 
U_j^2 V_i^2
\mathbb{E} [ \Pi_{ki}^2\Pi_{kj}^2 ] 
 = \frac{1}{m} U_j^2 V_i^2, \\
\mathbb{E} [ T_{ij3} | U_i, V_i, U_j, V_j] 
&=
\sum_{k=1}^{m} 
U_i U_j V_i V_j
\mathbb{E} [ \Pi_{ki}^2\Pi_{kj}^2 ] 
=
\frac{1}{m} U_i U_j V_i V_j. 
\end{align*}
Thus,
\begin{align*}
\mathbb{E} [ \{ H (W_i, W_j) \}^2 ] 
&=
\mathbb{E} [ \mathbb{E} [  \{ H (W_i, W_j) \}^2 | U_i, V_i, U_j, V_j] ] \\
&= \frac{1}{m}
\mathbb{E} \left(
U_i^2 V_j^2 + U_j^2 V_i^2 + U_i U_j V_i V_j
\right) \\
&= \frac{2}{m}
\{ \mathbb{E} ( U_i^2)  \mathbb{E} (V_i^2 ) + [\mathbb{E} ( U_i V_i )]^2 \}.
\end{align*}
Then, $\mathrm{Var} \left( T_{n2} \right)$ can be obtained by
combining the U statistic formula given in \eqref{Tn2-ustat} with {RP}(iii).
\end{proof}

\begin{proof}[Proof of Theorem~\ref{moments-thm}]
The theorem for {RS} follows immediately from Lemma~\ref{moments-unif-lem}
under the condition that $\sum_{i=1}^n p_i^2 = o (m^{-1})$.
The theorem for {RP} follows straightforwardly from Lemma~\ref{moments-RP-lem} using
Cauchy–Schwarz inequality.
We now consider {BS}.
Recall that 
\begin{align*}
n^{-1} \left( U^T \Pi^T \Pi V - U^T V \right) 
&= n^{-1} \sum_{i=1}^n    \left( \frac{n}{m} B_{ii} - 1 \right) U_i    V_i
\end{align*}
and that the summands are i.i.d.,
\begin{align*}
\mathbb{E} \left[
\left( \frac{n}{m} B_{ii} - 1 \right) U_i    V_i
 \right]
 = 0,
 \ \ \text{ and } \ \
 \mathrm{Var} \left[
\left( \frac{n}{m} B_{ii} - 1 \right) U_i    V_i
 \right]
= 
\left( \frac{n}{m} - 1 \right)
\mathbb{E} \left( U_i^2    V_i^2 \right).
 \end{align*}
Then, the desired result follows immediately.
\end{proof}

In order to prove Theorem~\ref{moments-thm-rp}(ii),
we use the central limit theorem for degenerate $U$-statistics of \citet{hall1984central}.
For the sake of easy referencing, we reproduce it below.

\begin{lemma}[Theorem 1 of \citet{hall1984central}]\label{lem:hall1984}
Assume that $\{ W_1, \ldots, W_n \}$ are independent and identically distributed random vectors. 
Define
\begin{align*}  
\mathbb{U}_n := 
\operatorname*{\sum \sum}_{1 \leq i < j \leq n} H_n (W_i, W_j).  
\end{align*}
Assume $H_n$ is symmetric, $\mathbb{E} [ H_n (W_1, W_2) | W_1 ] = 0$ almost surely
and $\mathbb{E} [ H_n^2 (W_1, W_2)  ] < \infty$ for each $n$. 
Let 
$
G_n (w_1, w_2) := \mathbb{E} [ H_n (W_1, w_1) H_n (W_1, w_2)  ]. 
$
If 
\begin{align*}
\frac{\mathbb{E}[ G_n^2 (W_1, W_2) ]  + n^{-1}  \mathbb{E} [ H_n^4 (W_1, W_2)  ] }
{ \{ \mathbb{E} [ H_n^2 (W_1, W_2)  ] \}^2 } \rightarrow 0
\end{align*}
as $n \rightarrow \infty$, then
\begin{align*}
\mathbb{V}_n^{-1/2} \frac{\mathbb{U}_n}{n} \rightarrow_d N(0, 1),
\end{align*}
where $\mathbb{V}_n := \frac{1}{2} \mathbb{E} [ H_n^2 (W_1, W_2)  ]$.
\end{lemma}

\begin{lemma}\label{lem:Gn-rate}
Let $\Pi$ be a random matrix satisfying {RP}.
Let 
$
G(w_1, w_2)  := \mathbb{E} [ H (W_i, w_1) H (W_i, w_2)  ]. 
$
Then, 
\begin{align*}
\mathbb{E} [ G^2(W_i, W_j)  ]
= O (m^{-3}).
\end{align*}
\end{lemma}

\begin{proof}
First, write
\begin{align*}
G(w_1, w_2) 
&= 
 \mathbb{E}
\left\{
\left[ \tilde{H} (W_i, w_1) + \tilde{H} (w_1, W_i) \right]
\left[ \tilde{H} (W_i, w_2) + \tilde{H} (w_2, W_i) \right]
\right\}.
\end{align*}
Because $\mathbb{E} [ \Pi_{ki}^2  ] = m^{-1}$ and  
$\mathbb{E} [ \Pi_{ki} \Pi_{\ell i} ] = 0$ whenever $k \neq \ell$ for each $i$, we have that
\begin{align*}
\mathbb{E} [ \tilde{H} (W_i, w_1) \tilde{H} (W_i, w_2) | U_i, V_i, U_j, V_j]
&=
\sum_{k=1}^{m} \sum_{j=1}^{m} 
U_i^2 v_{1} v_{2}
\mathbb{E} [ \Pi_{ki}  \Pi_{ji}   ]
\pi_{k1}   \pi_{j2}  \\
&=
\sum_{k=1}^{m} 
U_i^2 v_{1} v_{2}
\mathbb{E} [ \Pi_{ki}^2  ] \pi_{k1}    \pi_{k2} \\
&=
m^{-1} \sum_{k=1}^{m} 
U_i^2 v_{1} v_{2}
 \pi_{k1}    \pi_{k2}.
\end{align*}
Similarly,
\begin{align*}
\mathbb{E} [ \tilde{H} (w_1, W_i) \tilde{H} (w_2, W_i) | U_i, V_i, U_j, V_j]
&= 
m^{-1} \sum_{k=1}^{m} 
u_1 u_2  V_i^2
 \pi_{k1}    \pi_{k2}, \\
\mathbb{E} [ \tilde{H} (W_i, w_1) \tilde{H} (w_2, W_i) | U_i, V_i, U_j, V_j]
&= 
m^{-1} \sum_{k=1}^{m} 
U_i u_2  v_1 V_i
 \pi_{k1}    \pi_{k2}, \\
\mathbb{E} [ \tilde{H} (w_1, W_i) \tilde{H} (W_i, w_2) | U_i, V_i, U_j, V_j]
&= 
m^{-1} \sum_{k=1}^{m} 
u_1 U_i  V_i v_2
 \pi_{k1}    \pi_{k2}.
\end{align*}
Then, by simple algebra, 
\begin{align*}
G(W_i, W_j) &= 
m^{-1} \sum_{k=1}^{m} 
\left\{ \mathbb{E} [U_i^2] V_{i} V_{j}  +  \mathbb{E} [V_i^2]  U_{i} U_{j}
+ \mathbb{E} [U_i V_i]  U_{i} V_{j} + \mathbb{E} [U_i V_i]  U_{j} V_{i} \right\}
 \Pi_{ki}    \Pi_{kj}.
\end{align*}
Using  {RP}(ii)-(iii), write
\begin{align*}
\mathbb{E} [ G^2(W_i, W_j)]
&=
\mathbb{E} [\mathbb{E} [ G^2(W_i, W_j) | U_i, V_i, U_j, V_j] ] \\
&= 
m^{-2} \sum_{k=1}^{m} 
\mathbb{E}
\left[
\left\{ \mathbb{E} [U_i^2] V_{i} V_{j}  +  \mathbb{E} [V_i^2]  U_{i} U_{j}
+ \mathbb{E} [U_i V_i]  U_{i} V_{j} + \mathbb{E} [U_i V_i]  U_{j} V_{i} \right\}^2
\right]
\mathbb{E} [  \Pi_{ki}^2    \Pi_{kj}^2 ] \\
&= 
m^{-4} \sum_{k=1}^{m} 
\mathbb{E}
\left[
\left\{ \mathbb{E} [U_i^2] V_{i} V_{j}  +  \mathbb{E} [V_i^2]  U_{i} U_{j}
+ \mathbb{E} [U_i V_i]  U_{i} V_{j} + \mathbb{E} [U_i V_i]  U_{j} V_{i} \right\}^2 
\right] \\
&= O( m^{-3}),
\end{align*}
which proves the lemma.
\end{proof}

\begin{lemma}\label{lem:Fn-rate}
Let $\Pi$ be a random matrix satisfying {RP}.
Furthermore, assume that the columns of $\Pi$ are i.i.d.
Then, for $i \neq j$,
$\mathbb{E} [ \{ H (W_i, W_j) \}^4 ] = O (m^{-1})$.
\end{lemma}

\begin{proof}[Proof of Lemma~\ref{lem:Fn-rate}]
Using \eqref{H-notation-eq}, write
\begin{align*}
 \{ H (W_i, W_j) \}^4 
 = (T_{ij1} + T_{ij2} + 2 T_{ij3}) (T_{ij1} + T_{ij2} + 2 T_{ij3}).
 \end{align*}
 We expand the right-hand side of the equation above. 
 The first term has the form
 \begin{align*}
T_{ij1} T_{ij1}
&=
\sum_{k_1=1}^{m} \sum_{k_2 =1}^{m} \sum_{k_3=1}^{m} \sum_{k_4 =1}^{m}
U_i^4 V_j^4 \Pi_{k_1 i}  \Pi_{k_1 j} \Pi_{k_2 i}  \Pi_{k_2 j} \Pi_{k_3 i}  \Pi_{k_3 j} \Pi_{k_4 i}  \Pi_{k_4 j}. 
\end{align*}
 Combining {RP} with the additional assumption that the columns of $\Pi$ are i.id., we have that 
 $\mathbb{E} [ \Pi_{k i}^4  \Pi_{k j}^4  ] = O( m^{-2} )$ uniformly.
 Also,  
$\mathbb{E} [ \Pi_{k_1 i}  \Pi_{k_1 j} \Pi_{k_2 i}  \Pi_{k_2 j} \Pi_{k_3 i}  \Pi_{k_3 j} \Pi_{k_4 i}  \Pi_{k_4 j} ]$
is nonzero only if 
all four indices are the same ($k_1 = k_2 = k_3 = k_4 = k$)
or 
two pairs of the indices are the same (e.g., $k_1 =  k_2$ and $k_3 = k_4$).
This implies that
\begin{align*}
\mathbb{E} [ T_{ij1} T_{ij1}  ]
&=
\mathbb{E} \left( U_i^4 V_j^4 \right)
\left\{ 
\sum_{k=1}^m \mathbb{E} [ \Pi_{k i}^4  \Pi_{k j}^4] 
+ 
6 \sum_{k=1}^m \sum_{\ell =1, \ell \neq k}^m \mathbb{E} [ \Pi_{k i}^2  \Pi_{k j}^2 \Pi_{\ell i}^2  \Pi_{\ell j}^2] 
\right\}
= O (m^{-1}).
\end{align*}
Moreover, using similar arguments, we can show that 
all other terms 
$\mathbb{E} [ T_{ij k} T_{ij \ell}  ] = O (m^{-1})$, where
$k,\ell \in \{1,2,3\}$.   
Therefore, we have proved the lemma.
\end{proof}

\begin{lemma}\label{CLT-lemma}
Let $\Pi$ be a random matrix satisfying {RP}.
Furthermore, assume that the columns of $\Pi$ are i.i.d. Then, as $n \rightarrow \infty$, 
\begin{align*}
\sqrt{m} \, T_{n2}
\rightarrow_d 
N[0, \{ \mathbb{E} ( U_i^2)  \mathbb{E} (V_i^2 ) + \mathbb{E} ( U_i V_i )^2 \} ],
\end{align*}
\end{lemma}

\begin{proof}[Proof of Lemma~\ref{CLT-lemma}]
Note that $H (w_1, w_2) = H (w_2, w_1)$
and
$$
\mathbb{E} ( H (W_1, W_2) | W_1) = \mathbb{E} ( H (W_1, W_2) | W_2) = 0. 
$$
Thus, $T_{n2}$ is a degenerate $U$-statistic.
By Lemmas~\ref{moments-RP-lem},~\ref{lem:Gn-rate} and~\ref{lem:Fn-rate}, we have
 that 
\begin{align*}
\frac{\mathbb{E}[ G^2 (W_1, W_2) ]  + n^{-1}  \mathbb{E} [ H^4 (W_1, W_2)  ] }
{ \{ \mathbb{E} [ H^2 (W_1, W_2)  ] \}^2 } 
= O ( m^{-1} + n^{-1} m ) = o(1).
\end{align*}
Then, the conclusion of Lemma~\ref{CLT-lemma} follows directly by applying Lemma~\ref{lem:hall1984} along with Lemma~\ref{moments-RP-lem}.
\end{proof}

\begin{proof}[Proof of Theorem~\ref{moments-thm-rp}(ii)]
Recall that 
\begin{align*}
\tilde{\beta}_{OLS} - \hat{\beta}_{OLS} 
&= 
\Xi ( \tilde{\Upsilon}_n - \hat \Upsilon_n)
+ 
o_p ( m^{-1/2}  ).
\end{align*}
Then, the theorem follows immediately
by applying Lemma~\ref{CLT-lemma} with $U_i = c^T X_i$ and $V_i = e_i$ for each constant vector $c \in \mathbb{R}^p$.
\end{proof}

\section{Appendix: Proofs for 2SLS}\label{sec:proofs}


\subsection{Proof of Theorem~\ref{main-thm}}

Recall that 
using the singular value decomposition of $X$ and ${Z}$, we write 
${X} = U_X \Sigma_X V_X^T$
and
${Z} = U_Z \Sigma_Z V_Z^T$.
Define
\begin{align}\label{def-uhat}
\hat{\theta} := \left(  U_X^T U_Z  U_Z^T U_X \right)^{-1}
 U_X^T U_Z  U_Z^T y
\end{align}
and 
\begin{align}\label{def-utilde}
\tilde{\theta} := \left(  U_X^T  \Pi^T \Pi U_Z  \left( U_Z^T \Pi^T \Pi U_Z \right)^{-1} U_Z^T \Pi^T \Pi  U_X \right)^{-1} 
U_X^T  \Pi^T \Pi U_Z  \left( U_Z^T \Pi^T \Pi U_Z \right)^{-1} U_Z^T \Pi^T 
\Pi y.
\end{align}
It would be convenient to work with $U_X \hat{\theta}$ and $U_X \tilde{\theta}$ in order to analyze algorithmic properties 
of sketched 2SLS estimators because $U_X$ is an orthonormal matrix. 
The following lemma establishes the equivalence between $X \hat{\beta}$ and  $U_X \hat{\theta}$.

\begin{lemma}\label{lem-uhat}
Let Assumption \ref{embed} hold. Then,
$X \hat{\beta} = U_X \hat{\theta}$. 
\end{lemma}

\begin{proof}
By the singular value decomposition of $X$ and ${Z}$, we have that 
\begin{align*}
{Z}^T  {Z} &= V_Z \Sigma_Z^2  V_Z^T, \\
{Z} ( {Z}^T {Z} )^{-1} {Z}^T &=
U_Z  U_Z^T, \\
X^T {Z} ( {Z}^T {Z} )^{-1} {Z}^T X &=
V_X \Sigma_X U_X^T U_Z  U_Z^T U_X \Sigma_X V_X^T, \\
\left ( X^T {Z} ( {Z}^T {Z} )^{-1} {Z}^T X \right)^{-1} &=
V_X \Sigma_X^{-1} \left(  U_X^T U_Z  U_Z^T U_X \right)^{-1} \Sigma_X^{-1} V_X^T, \\
X^T {Z} ( {Z}^T {Z} )^{-1} {Z}^T y &=
V_X \Sigma_X U_X^T U_Z  U_Z^T y.
\end{align*}
Therefore,
\begin{align*}
\hat{\beta} &=
 V_X \Sigma_X^{-1} \left(  U_X^T U_Z  U_Z^T U_X \right)^{-1}
 U_X^T U_Z  U_Z^T y, \\
 X \hat{\beta} &=
 U_X \left(  U_X^T U_Z  U_Z^T U_X \right)^{-1}
 U_X^T U_Z  U_Z^T y,
\end{align*}
which in turn implies the conclusion in view of the definition of $\hat{\theta}$ in \eqref{def-uhat}.
\end{proof}

As in Lemma \ref{lem-uhat}, 
the equivalence between $X \tilde{\beta}$ and $U_X \tilde{\theta}$ holds.

\begin{lemma}\label{lem-utilde}
Assume that 
(i) $\tilde{{Z}}^T \tilde{{Z}}$ is non-singular and 
(ii) $\tilde{{X}}^T \tilde{{Z}} ( \tilde{{Z}}^T \tilde{{Z}} )^{-1} \tilde{{Z}}^T \tilde{{X}}$ is non-singular.
Then,
$X \tilde{\beta} = U_X \tilde{\theta}$. 
\end{lemma}

\begin{proof}
As in the proof of Lemma \ref{lem-uhat}, we have that  
\begin{align*}
\tilde{Z}^T  \tilde{Z} &= V_Z \Sigma_Z U_Z^T \Pi^T \Pi U_Z \Sigma  V_Z^T, \\
\left( \tilde{Z}^T  \tilde{Z} \right)^{-1} &=  V_Z \Sigma_Z^{-1} \left( U_Z^T \Pi^T \Pi U_Z \right)^{-1} \Sigma_Z^{-1} V_Z^T, \\
\tilde{Z} ( \tilde{Z}^T \tilde{Z} )^{-1} \tilde{Z}^T &=
\Pi U_Z  \left( U_Z^T \Pi^T \Pi U_Z \right)^{-1} U_Z^T \Pi^T, \\
\tilde{X}^T \tilde{Z} ( \tilde{Z}^T \tilde{Z} )^{-1} \tilde{Z}^T \tilde{X} &=
V_X \Sigma_X 
U_X^T  \Pi^T \Pi U_Z  \left( U_Z^T \Pi^T \Pi U_Z \right)^{-1} U_Z^T \Pi^T \Pi  U_X 
\Sigma_X V_X^T, \\
\left ( \tilde{X}^T \tilde{Z} ( \tilde{Z}^T \tilde{Z} )^{-1} \tilde{Z}^T \tilde{X} \right)^{-1} &=
V_X \Sigma_X^{-1} 
\left(  U_X^T  \Pi^T \Pi U_Z  \left( U_Z^T \Pi^T \Pi U_Z \right)^{-1} U_Z^T \Pi^T \Pi  U_X \right)^{-1} 
\Sigma_X^{-1} V_X^T, \\
\tilde{X}^T \tilde{Z} ( \tilde{Z}^T \tilde{Z} )^{-1} \tilde{Z}^T \tilde{y} &=
V_X \Sigma_X U_X^T \Pi^T
\Pi U_Z  \left( U_Z^T \Pi^T \Pi U_Z \right)^{-1} U_Z^T \Pi^T 
\Pi y.
\end{align*}
Therefore,
\begin{align*}
\tilde{\beta} &=
 V_X \Sigma_X^{-1} \left(  U_X^T  \Pi^T \Pi U_Z  \left( U_Z^T \Pi^T \Pi U_Z \right)^{-1} U_Z^T \Pi^T \Pi  U_X \right)^{-1} 
U_X^T  \Pi^T \Pi U_Z  \left( U_Z^T \Pi^T \Pi U_Z \right)^{-1} U_Z^T \Pi^T 
\Pi y, \\
 X \tilde{\beta} &=
 U_X \left(  U_X^T  \Pi^T \Pi U_Z  \left( U_Z^T \Pi^T \Pi U_Z \right)^{-1} U_Z^T \Pi^T \Pi  U_X \right)^{-1} 
U_X^T  \Pi^T \Pi U_Z  \left( U_Z^T \Pi^T \Pi U_Z \right)^{-1} U_Z^T \Pi^T 
\Pi y,
\end{align*}
which again implies the conclusion in view of the definition of $\tilde{\theta}$ in \eqref{def-utilde}.
\end{proof}

Abusing the notation a bit, 
define now
\begin{align*}
\tilde{A} &:= U_X^T  \Pi^T \Pi U_Z  \left( U_Z^T \Pi^T \Pi U_Z \right)^{-1} U_Z^T \Pi^T \Pi  U_X, \\
\hat{A} &:= U_X^T U_Z  U_Z^T U_X, \\
\tilde{B} &:= U_X^T  \Pi^T \Pi U_Z  \left( U_Z^T \Pi^T \Pi U_Z \right)^{-1} U_Z^T \Pi^T 
\Pi \hat{e}, \\
\hat{B} &:= U_X^T U_Z  U_Z^T \hat{e}.
\end{align*}

\begin{lemma}\label{lem-uhat-zero}
Let Assumption \ref{embed} hold. Then,
$\hat{A}^{-1} \hat{B}  = 0$. 
\end{lemma}

\begin{proof}
Note that 
\begin{align*}
\hat{A}^{-1} \hat{B} 
&= 
\left( U_X^T U_Z  U_Z^T U_X \right)^{-1} U_X^T U_Z  U_Z^T \hat{e} \\
&= 
\left( U_X^T U_Z  U_Z^T U_X \right)^{-1} U_X^T U_Z  U_Z^T y
-
\left( U_X^T U_Z  U_Z^T U_X \right)^{-1} U_X^T U_Z  U_Z^T {X} \hat{\beta} \\
&= 0,
\end{align*}
since 
${X} \hat{\beta} = 
U_X \left(  U_X^T U_Z  U_Z^T U_X \right)^{-1}
 U_X^T U_Z  U_Z^T y$.
\end{proof}

Under Assumption \ref{embed}, we first obtain the following lemma.

\begin{lemma}\label{tilde-A-B}
Let Assumption \ref{embed} hold. Then,
the following holds jointly with probability at least $1-\delta:$
\begin{align*}
\left\|  \tilde{A} - \hat{A} \right\|_2 &\leq f_1 (\varepsilon_1,  \varepsilon_2), \\ 
\left\|  \tilde{B} - \hat{B} \right\|_2 &\leq \varepsilon_3 \left\|  \hat{e} \right\| 
+
f_2 (\varepsilon_1,  \varepsilon_2)
\left[ 1 + \varepsilon_3 \left\|  \hat{e} \right\|  \right].
\end{align*}
\end{lemma}

\begin{proof}
Let 
$\tilde{A}_1 := U_Z^T \Pi^T \Pi  U_X$, 
$\tilde{A}_2 := \left( U_Z^T \Pi^T \Pi U_Z \right)^{-1}$,
$\hat{A}_1 := U_Z^T U_X$,
and
$\hat{A}_2 := I$.
Then we have that 
\begin{align*}
&\tilde{A} - \hat{A} \\ 
& = \tilde{A}_1^T \tilde{A}_2 \tilde{A}_1  -  \hat{A}_1^T \hat{A}_2 \hat{A}_1  \\
&=
 (\tilde{A}_1 - \hat{A}_1)^T \tilde{A}_2 (\tilde{A}_1 - \hat{A}_1)  
 +   \hat{A}_1^T \tilde{A}_2  (\tilde{A}_1 - \hat{A}_1) 
 +   (\tilde{A}_1 - \hat{A}_1)^T \tilde{A}_2    \hat{A}_1
 +  \hat{A}_1^T (\tilde{A}_2 - \hat{A}_2) \hat{A}_1 \\
&=
(\tilde{A}_1 - \hat{A}_1)^T \hat{A}_2 (\tilde{A}_1 - \hat{A}_1)  
+ (\tilde{A}_1 - \hat{A}_1)^T (\tilde{A}_2 - \hat{A}_2) (\tilde{A}_1 - \hat{A}_1) \\
&+   \hat{A}_1^T \hat{A}_2  (\tilde{A}_1 - \hat{A}_1) 
+ \hat{A}_1^T (\tilde{A}_2 - \hat{A}_2)  (\tilde{A}_1 - \hat{A}_1) \\
&+  (\tilde{A}_1 - \hat{A}_1)^T \hat{A}_2    \hat{A}_1
+  (\tilde{A}_1 - \hat{A}_1)^T (\tilde{A}_2 - \hat{A}_2)    \hat{A}_1 \\
&+ \hat{A}_1^T (\tilde{A}_2 - \hat{A}_2) \hat{A}_1.
 \end{align*}
It is straightforward to show that $\norm{ \tilde{A}_2 - \hat{A}_2 }_2 \leq \varepsilon_1/(1-\varepsilon_1)$ using Assumption \ref{embed}(i).
Since $\left\|  \hat{A}_1 \right\|_2 \leq \left\|  U_Z \right\|_2  \left\|  U_X \right\|_2 = 1$
and $\left\|  \hat{A}_2 \right\|_2 = 1$,
we have that 
\begin{align*}
\left\|  \tilde{A} - \hat{A} \right\|_2 
&\leq
\varepsilon_2^2 + \varepsilon_2^2 \varepsilon_1/(1-\varepsilon_1) 
+ 2\varepsilon_2 + 2\varepsilon_2 \varepsilon_1/(1-\varepsilon_1)
+ \varepsilon_1/(1-\varepsilon_1) \\
&= \frac{\varepsilon_1 + \varepsilon_2 (\varepsilon_2  + 2)}{1-\varepsilon_1} = f_1 (\varepsilon_1,  \varepsilon_2), 
\end{align*}
using Assumption \ref{embed}. This proves the first desired result.

Now 
let 
$\tilde{B}_1 := U_Z^T \Pi^T 
\Pi \hat{e}$ and $\hat{B}_1 := U_Z^T \hat{e}$.
Consider
\begin{align*}
\tilde{B} - \hat{B}
&= 
U_X^T  \Pi^T \Pi U_Z  \left( U_Z^T \Pi^T \Pi U_Z \right)^{-1} U_Z^T \Pi^T 
\Pi \hat{e}
- 
U_X^T U_Z  U_Z^T  \hat{e} \\
& = \tilde{A}_1^T \tilde{A}_2 \tilde{B}_1  -  \hat{A}_1^T \hat{A}_2 \hat{B}_1  \\
&=
 (\tilde{A}_1 - \hat{A}_1)^T \tilde{A}_2 (\tilde{B}_1 - \hat{B}_1)  
 +   \hat{A}_1^T \tilde{A}_2  (\tilde{B}_1 - \hat{B}_1) 
 +   (\tilde{A}_1 - \hat{A}_1)^T \tilde{A}_2    \hat{B}_1
 +  \hat{A}_1^T (\tilde{A}_2 - \hat{A}_2) \hat{B}_1 \\
&=
(\tilde{A}_1 - \hat{A}_1)^T \hat{A}_2 (\tilde{B}_1 - \hat{B}_1)  
+ (\tilde{A}_1 - \hat{A}_1)^T (\tilde{A}_2 - \hat{A}_2) (\tilde{B}_1 - \hat{B}_1) \\
&+   \hat{A}_1^T \hat{A}_2  (\tilde{B}_1 - \hat{B}_1) 
+ \hat{A}_1^T (\tilde{A}_2 - \hat{A}_2)  (\tilde{B}_1 - \hat{B}_1) \\
&+  (\tilde{A}_1 - \hat{A}_1)^T \hat{A}_2    \hat{B}_1
+  (\tilde{A}_1 - \hat{A}_1)^T (\tilde{A}_2 - \hat{A}_2)    \hat{B}_1 \\
&+ \hat{A}_1^T (\tilde{A}_2 - \hat{A}_2) \hat{B}_1.
\end{align*}
Since 
$\left\| \hat{B}_1 \right\|_2
= \left\| U_Z^T \hat{e} \right\|_2 \leq \left\| U_Z \right\|_2 \left\|  \hat{e} \right\|
\leq \left\|  \hat{e} \right\|$, 
we have that 
\begin{align*}
\left\|  \tilde{B} - \hat{B} \right\|_2 
&\leq
\varepsilon_3 \left\|  \hat{e} \right\| 
+
\left[ \varepsilon_2 + \varepsilon_1/(1-\varepsilon_1)  + \varepsilon_2 \varepsilon_1/(1-\varepsilon_1) \right] 
\left[ 1 + \varepsilon_3 \left\|  \hat{e} \right\|  \right]
\\
&= \varepsilon_3 \left\|  \hat{e} \right\| 
+
f_2 (\varepsilon_1,  \varepsilon_2)
\left[ 1 + \varepsilon_3 \left\|  \hat{e} \right\|  \right],
\end{align*}
again using Assumption \ref{embed}.
This proves the second desired result.
\end{proof}

\begin{lemma}\label{lem-inverse}
Let Assumptions \ref{embed}  hold. 
Then, the following holds  with probability at least $1-\delta:$
\begin{align*}
\sigma_{\min} (\tilde{A}) \geq \frac{1}{2} \sigma_{\min}^2 (   U_Z^T U_X ).
\end{align*}

\begin{proof}
Use the fact that for real matrices $C$ and $D$,
\begin{align*}
\sigma_{\min} (C + D) \geq \sigma_{\min} (C)  - \sigma_{\max} (D)
\end{align*}
to obtain
\begin{align*}
\sigma_{\min} (\tilde{A}) \geq \sigma_{\min} (\hat{A})  - \sigma_{\max} (\tilde{A} - \hat{A}).
\end{align*}
Then the desired result follows from the first conclusion of Lemma \ref{tilde-A-B}, since
\begin{align*}
\sigma_{\min} (\hat{A}) 
= \sigma_{\min} ( U_X^T U_Z  U_Z^T U_X ) = \sigma_{\min}^2 (   U_Z^T U_X )
\ \ \text{ and }  \ \
\sigma_{\max} (\tilde{A} - \hat{A})
\leq \left\|  \tilde{A} - \hat{A} \right\|_2.
\end{align*}
\end{proof}
\end{lemma}

Lemma \ref{lem-inverse} implies that 
$\tilde{A}^{-1}$ is well defined with probability at least $1-\delta$.

\begin{lemma}\label{lem:A-inverse}
Let Assumptions \ref{embed}  hold. 
Then, the following holds  with probability at least $1-\delta:$
\begin{align*}
\norm{ \tilde{A}^{-1} - \hat{A}^{-1} }_2 &\leq  \frac{2 f_1 (\varepsilon_1,  \varepsilon_2)}{ \sigma_{\min}^4 ( U_Z^T U_X ) }.
\end{align*}
\end{lemma}

\begin{proof}
Write
\begin{align*}
\tilde{A}^{-1} - \hat{A}^{-1} 
&=
\hat{A}^{-1} \left( \hat{A} - \tilde{A} \right) \tilde{A}^{-1}.
\end{align*}
Thus,
\begin{align*}
\norm{ \tilde{A}^{-1} - \hat{A}^{-1} }_2
&\leq
\norm{ \hat{A}^{-1} }_2 \norm{ \hat{A} - \tilde{A} }_2 \norm{ \tilde{A}^{-1} }_2 \\
&\leq  \frac{2 f_1 (\varepsilon_1,  \varepsilon_2)}{ \sigma_{\min}^4 ( U_Z^T U_X ) }
\end{align*}
since 
$\norm{ \hat{A}^{-1} }_2 =  \left[ \sigma_{\min}^2 ( U_Z^T U_X ) \right]^{-1}$,
by Lemma \ref{tilde-A-B}, 
$\norm{ \hat{A} - \tilde{A} }_2 \leq f_1 (\varepsilon_1,  \varepsilon_2)$
and, by Lemma \ref{lem-inverse},
$\norm{  \tilde{A}^{-1} }_2 \leq  2 \left[ \sigma_{\min}^2 ( U_Z^T U_X ) \right]^{-1}$
with probability at least $1-\delta$.
\end{proof}

\begin{proof}[Proof of Theorem \ref{main-thm}]
By Lemmas \ref{lem-uhat} and  \ref{lem-utilde},
\begin{align*}
X(\tilde{\beta} - \hat{\beta}) = U_X (\tilde{\theta} - \hat{\theta}),
\end{align*}
so that 
\begin{align*}
\sigma_{\min}(X)  \norm{ \tilde{\beta} - \hat{\beta} }  
\leq \norm{ \tilde{\theta} - \hat{\theta}}.
\end{align*}
Thus, it suffices to bound $\norm{ \tilde{\theta} - \hat{\theta}}$.
To do so, write 
\begin{align}\label{y-expression}
\tilde{{y}} = \Pi \left( {X} \hat{\beta} +  \hat{{e}} \right) = 
\tilde{{X}} \hat{\beta} + \tilde{{e}}
= \Pi U_X \hat{\theta} + \tilde{{e}}, 
\end{align}
where $\tilde{{e}} = \Pi \hat{{e}}$.
Plugging \eqref{y-expression} into \eqref{def-utilde} yields
\begin{align*}
\tilde{\theta} - \hat{\theta} 
&=  \tilde{A}^{-1} \tilde{B}.
\end{align*}
Then, by Lemma \ref{lem-uhat-zero}, we have that 
$
\tilde{\theta} - \hat{\theta} 
=  \tilde{A}^{-1} \tilde{B}  
=  \tilde{A}^{-1} \tilde{B} - \hat{A}^{-1} \hat{B}.
$
Further, write
\begin{align*}
\tilde{\theta} - \hat{\theta} 
= 
\left( \tilde{A}^{-1} - \hat{A}^{-1} \right) \hat{B} 
+ \hat{A}^{-1} \left( \tilde{B} - \hat{B} \right)
+ \left( \tilde{A}^{-1} - \hat{A}^{-1} \right) \left( \tilde{B} - \hat{B} \right).
\end{align*}
Thus,
\begin{align*}
\| \tilde{\theta} - \hat{\theta} \|
&= \| \tilde{A}^{-1} \tilde{B} - \hat{A}^{-1} \hat{B} \|_2 \\
&\leq 
\norm{ \tilde{A}^{-1} - \hat{A}^{-1} }_2 \norm{ \hat{B} }_2 
+ \norm{ \hat{A}^{-1} } \norm{ \tilde{B} - \hat{B} }_2 
+ \norm{ \tilde{A}^{-1} - \hat{A}^{-1} } \norm{ \tilde{B} - \hat{B} }_2 \\
&\leq  
\frac{2 f_1 (\varepsilon_1,  \varepsilon_2)}{ \sigma_{\min}^4 ( U_Z^T U_X ) } \left\|  \hat{e} \right\|
+ \frac{\varepsilon_3 \left\|  \hat{e} \right\| 
+
f_2 (\varepsilon_1,  \varepsilon_2)
\left[ 1 + \varepsilon_3 \left\|  \hat{e} \right\|  \right]}{ \sigma_{\min}^2 ( U_Z^T U_X ) } \\
&+ \frac{2 f_1 (\varepsilon_1,  \varepsilon_2) }{ \sigma_{\min}^4 ( U_Z^T U_X ) } 
\left\{   
\varepsilon_3 \left\|  \hat{e} \right\| 
+
f_2 (\varepsilon_1,  \varepsilon_2)
\left[ 1 + \varepsilon_3 \left\|  \hat{e} \right\| \right]
\right\}
 \\
 &= \frac{f_2 (\varepsilon_1,  \varepsilon_2) + \varepsilon_3 \left\|  \hat{e} \right\| 
\left[ 1 +
f_2 (\varepsilon_1,  \varepsilon_2)   \right]}{ \sigma_{\min}^2 ( U_Z^T U_X ) }
 \left[ 1+ \frac{2 f_1 (\varepsilon_1,  \varepsilon_2)}{ \sigma_{\min}^2 ( U_Z^T U_X ) } \right],
\end{align*}
where the last inequality follows from Assumption \ref{embed}.
\end{proof}

We now specialize Theorem~\ref{main-thm} for countsketch.

\begin{theorem}\label{main-thm:cs}
Let data $\mathcal{D}_n$ be fixed, $Z^TZ$ and $X^T P_Z X$ are non-singular.
Let
$\Pi \in \mathbb{R}^{m \times n}$ be counsketch
with $m \geq \max \{ q(q+1),  2 p q \} /(\varepsilon^2 \delta)$
for some $\varepsilon \in (0,1/3]$.
Suppose that 
$\sigma_{\min}^2 ( U_Z^T U_X ) 
\geq \frac{16 \varepsilon(1+\varepsilon)}{1-\varepsilon}.$
and let $\underline \sigma*=
\left[ \sigma_{\min}(X) \sigma_{\min}^2 ( U_Z^T U_X ) \right]^{-1}. $ Then, the following holds  with probability at least $1-\delta:$
\begin{align*}
\norm{ \tilde{\beta}_{2SLS} - \hat{\beta}_{2SLS} }
& \leq \frac{4 \varepsilon}{1-\varepsilon}\left[  2 + 3 \frac{\left\|  \hat{e} \right\|}{\sqrt{p}}   \right] \underline\sigma^*
 \end{align*}
\end{theorem}

To establish Lemma~\ref{lem:suff:embed:cs} given below, we first state the following known results in the literature.

\begin{lemma}[Theorem 6.2 of \citet{KaneNelson2014}]
Distribution $\mathcal{D}$ over $\mathbb{R}^{m \times n}$ is defined to have $(\varepsilon, \delta, 2)$-JL (Johnson-Lindenstrauss) moments 
if for any $x \in \mathbb{R}^{n}$ such that $\| x \| = 1$, 
\begin{align*}
\mathbb{E}_{\Pi \sim \mathcal{D}} 
\left[ \left| 
\| \Pi x \|^2 - 1
\right|^2 
\right]
\leq \varepsilon^2 \delta.
\end{align*} 
Given $\varepsilon, \delta \in (0, 1/2)$, let $\mathcal{D}$ be any distribution over matrices with $n$ 
columns with the $(\varepsilon, \delta, 2)$-JL moment property. Then,
for any $A$ and $B$ real matrices each with $n$ rows,
\begin{align*}
\mathbb{P}_{\Pi \sim \mathcal{D}} 
\left( 
\| A^T \Pi^T \Pi B - A^T B \|_F > 3 \varepsilon \| A \|_F \| B \|_F 
\right)
< \delta.
\end{align*}
\end{lemma}

\begin{lemma}[Theorem 2.9 of \citet{woodruff2014sketching}]
Let $\Pi \in \mathbb{R}^{m \times n}$ be countsketch with $m \geq 2/(\varepsilon^2 \delta)$. 
Then, $\Pi$ satisfies the $(\varepsilon, \delta, 2)$-JL moment property. 
\end{lemma}

\begin{lemma}\label{lem:suff:embed:cs}
Let $\Pi \in \mathbb{R}^{m \times n}$ be countsketch 
with $m \geq \max \{ q(q+1),  2 p q \} /(\varepsilon^2 \delta)$
for some $\varepsilon \in (0,1/2)$.
Then, Assumption~\ref{embed} holds with $\varepsilon_1 = \varepsilon, \varepsilon_2 = 3 \varepsilon, \varepsilon_3 = 3 \varepsilon p^{-1/2}$. 
\end{lemma}

\begin{proof}[Proof of Lemma~\ref{lem:suff:embed:cs}]
As shown in the proof of Theorem 2 of \citet{NelsonNguyen2013}, 
\begin{align*}
 \mathbb{P}_{\Pi \sim \mathcal{D}}  \left(  \| U_Z^T \Pi^T \Pi U_Z - I_q \|_2 >  \varepsilon  \right) & < \delta,
\end{align*}
provided that $m \geq q(q+1)/(\varepsilon^2 \delta)$.
This verifies the first condition of Assumption~\ref{embed}.

Now to verify conditions (ii) and (iii) of Assumption \ref{embed}, note that since countsketch with $m \geq 2/(\varepsilon^2 \delta)$ satisfies the $(\varepsilon, \delta, 2)$-JL moment property, we have, for any any $A$ and $B$ real matrices each with $n$ rows,
\begin{align*}
& \mathbb{P}_{\Pi \sim \mathcal{D}} 
\left( 
\| A^T \Pi^T \Pi B - A^T B \|_2 > 3 \varepsilon \| A \|_F \| B \|_F 
\right) \\
&\leq
\mathbb{P}_{\Pi \sim \mathcal{D}} 
\left( 
\| A^T \Pi^T \Pi B - A^T B \|_F > 3 \varepsilon \| A \|_F \| B \|_F 
\right)
< \delta.
\end{align*}
Since 
$ \| U_X \|_F^2 = p$,
$ \| U_Z \|_F^2 = q$
and $ \| \hat{e} \|_F =  \| \hat{e} \|$, 
we have that
\begin{align*}
 \mathbb{P}_{\Pi \sim \mathcal{D}}  \left(  \| U_Z^T \Pi^T \Pi U_X - U_Z^T U_X \|_2 > 3 \varepsilon \sqrt{p q} \right) & < \delta, \\
  \mathbb{P}_{\Pi \sim \mathcal{D}}  \left(  \| U_Z^T \Pi^T 
\Pi \hat{e} -  U_Z^T \hat{e} \| > 3 \varepsilon \sqrt{q} \| \hat{e} \|  \right) & < \delta,
\end{align*}
provided that $m \geq 2/(\varepsilon^2 \delta)$.
Replacing $\varepsilon$ with $\varepsilon/\sqrt{pq}$ yields that 
\begin{align*}
 \mathbb{P}_{\Pi \sim \mathcal{D}}  \left(  \| U_Z^T \Pi^T \Pi U_X - U_Z^T U_X \|_2 >  3 \varepsilon  \right) & < \delta, \\
  \mathbb{P}_{\Pi \sim \mathcal{D}}  \left(  \| U_Z^T \Pi^T 
\Pi \hat{e} -  U_Z^T \hat{e} \| >  3 \varepsilon p^{-1/2}  \| \hat{e} \|  \right) & < \delta,
\end{align*}
provided that $m \geq 2 p q/(\varepsilon^2 \delta)$.
Thus, we have proved Lemma~\ref{lem:suff:embed:cs}.
\end{proof}

\begin{proof}[Proof of Theorem~\ref{main-thm:cs}]
In view of Lemma \ref{lem:suff:embed:cs},
this theorem follows directly from applying Theorem~\ref{main-thm} to the case when $\Pi$ is a countsketch. 
\end{proof}

\subsection{Proof of Theorem~\ref{main-CLT-thm}}

\begin{proof}[Proof of Theorem~\ref{main-CLT-thm}]
It follows from the definition of the estimator that 
\begin{align*}
\tilde{\beta} &= \left( \tilde{{X}}^T \tilde{{Z}} ( \tilde{{Z}}^T \tilde{{Z}} )^{-1} \tilde{{Z}}^T \tilde{{X}} \right)^{-1} 
\tilde{{X}}^T \tilde{{Z}} ( \tilde{{Z}}^T \tilde{{Z}} )^{-1} \tilde{{Z}}^T 
\left( \tilde{{X}} \beta_0 +  \tilde{{e}} \right) \\
&= \beta_0
+
 \left\{ (\tilde{{X}}^T \tilde{{Z}}/n) ( \tilde{{Z}}^T \tilde{{Z}}/n )^{-1} (\tilde{{Z}}^T \tilde{{X}}/n) \right\}^{-1} 
(\tilde{{X}}^T \tilde{{Z}}/n) ( \tilde{{Z}}^T \tilde{{Z}}/n )^{-1} (\tilde{{Z}}^T 
  \tilde{{e}}/n).
\end{align*}
Thus,,
\begin{align*}
\hat{\beta} =  \beta_0 + 
\left[ ({X}^T {Z}/n) ( {Z}^T {Z}/n )^{-1} ({Z}^T {X}/n)  \right]^{-1} 
({X}^T {Z}/n ( {Z}^T {Z}/n )^{-1} {Z}^T {e}/n.
\end{align*}
Write
\begin{align*}
\tilde{\beta} - \hat{\beta} 
&= (\tilde{\Xi}_n - \Xi) \Upsilon_n + \Xi ( \tilde{\Upsilon}_n - \Upsilon_n) 
+ (\tilde{\Xi}_n - \Xi) ( \tilde{\Upsilon}_n - \Upsilon_n), 
\end{align*}
where 
$\tilde{\Upsilon}_n = \tilde{{Z}}^T   \tilde{{e}}/n$,
$\Upsilon_n = {Z}^T {e}/n$,
\begin{align*}
\tilde{\Xi}_n &=
 \left\{ (\tilde{{X}}^T \tilde{{Z}}/n) ( \tilde{{Z}}^T \tilde{{Z}}/n )^{-1} (\tilde{{Z}}^T \tilde{{X}}/n) \right\}^{-1} 
(\tilde{{X}}^T \tilde{{Z}}/n) ( \tilde{{Z}}^T \tilde{{Z}}/n )^{-1}, \\
\Xi
&=
\left[ \mathbb{E} ( X_i Z_i^T ) \left[ \mathbb{E} ( Z_i Z_i^T ) \right]^{-1} \mathbb{E} ( Z_i X_i^T ) \right]^{-1}
\mathbb{E} ( X_i Z_i^T ) \left[ \mathbb{E} ( Z_i Z_i^T ) \right]^{-1}.
\end{align*}
As in the proof for OLS,  
$\tilde{\Xi}_n - \Xi = o_p (1)$ 
and
$\tilde{\Upsilon}_n - \Upsilon_n = O_p (m^{-1/2})$.
By the central limit theorem, 
$\Upsilon_n = O_p (n^{-1/2})$.
Hence,
\begin{align*}
\tilde{\beta} - \hat{\beta} 
&= \Xi ( \tilde{\Upsilon}_n - \Upsilon_n)  
+ o_p \left( n^{-1/2} + m^{-1/2} \right).
\end{align*}
Moreover, by Lemma~\ref{CLT-lemma} that 
\begin{align*}
m^{1/2} ( \tilde{\Upsilon}_n - \Upsilon_n)  
\rightarrow_d N [0, \mathbb{E} (e_i^2 Z_i Z_i^T )].
\end{align*}
Combining all the arguments above yields 
\begin{align*}
m^{1/2} ( \tilde{\beta} - \hat{\beta} 
)  
\rightarrow_d N [0, \Xi  \mathbb{E} (e_i^2 Z_i Z_i^T ) \Xi^T],
\end{align*}
which gives the conclusion of the theorem.
\end{proof}

\bibliography{TSLS}

\end{document}